\documentclass[letterpaper]{article} % DO NOT CHANGE THIS
\usepackage{aaai2026}  % DO NOT CHANGE THIS
\usepackage{times}  % DO NOT CHANGE THIS
\usepackage{helvet}  % DO NOT CHANGE THIS
\usepackage{courier}  % DO NOT CHANGE THIS
\usepackage[hyphens]{url}  % DO NOT CHANGE THIS
\usepackage{graphicx} % DO NOT CHANGE THIS
\usepackage{natbib}  % DO NOT CHANGE THIS AND DO NOT ADD ANY OPTIONS TO IT
\usepackage{caption} % DO NOT CHANGE THIS AND DO NOT ADD ANY OPTIONS TO IT
\usepackage{algorithm}
\usepackage{algorithmic}

\usepackage{newfloat}
\usepackage{listings}
\DeclareCaptionStyle{ruled}{labelfont=normalfont,labelsep=colon,strut=off} % DO NOT CHANGE THIS
\floatstyle{ruled}
\newfloat{listing}{tb}{lst}{}
\floatname{listing}{Listing}
\usepackage{mathtools}
\usepackage{amsthm}
\usepackage[capitalize,noabbrev]{cleveref}
\usepackage{multirow}
\usepackage{tabularx}
\usepackage{booktabs}
\usepackage{amsmath,amsfonts,bm}

\def\eqref#1{equation~\ref{#1}}
\def\1{\bm{1}}

\def\vtheta{{\bm{\theta}}}

\def\vg{{\bm{g}}}

\def\vx{{\bm{x}}}

\def\mA{{\bm{A}}}

\def\mI{{\bm{I}}}

\def\mP{{\bm{P}}}

\DeclareMathAlphabet{\mathsfit}{\encodingdefault}{\sfdefault}{m}{sl}
\SetMathAlphabet{\mathsfit}{bold}{\encodingdefault}{\sfdefault}{bx}{n}

\newcommand{\E}{\mathbb{E}}

\newcommand{\R}{\mathbb{R}}

\theoremstyle{plain}
\newtheorem{theorem}{Theorem}
\newtheorem*{formalthm}{Theorem \ref{thm:fedpm_informal}}

\newtheorem{assumption}{Assumption}
\newtheorem{condition}{Condition}
\usepackage{comment}

\title{FedPM: Federated Learning Using Second-order Optimization\\ with Preconditioned Mixing of Local Parameters}
\author {
    Hiro Ishii\textsuperscript{\rm 1},
    Kenta Niwa\textsuperscript{\rm 2},
    Hiroshi Sawada\textsuperscript{\rm 2},
    Akinori Fujino\textsuperscript{\rm 2},
    Noboru Harada\textsuperscript{\rm 2},  
    Rio Yokota\textsuperscript{\rm 1}
}
\affiliations {
    \textsuperscript{\rm 1}Institute of Science Tokyo, 
    \textsuperscript{\rm 2}NTT Communication Science Laboratories\\
    \{ishii.h, rioyokota\}@rio.scrc.iir.isct.ac.jp, \{kenta.niwa, hrsh.sawada, harada.noboru\}@ntt.com
}

\begin{document}

\maketitle
\begin{abstract}
We propose Federated Preconditioned Mixing (FedPM), a novel Federated Learning (FL) method that leverages second-order optimization. Prior methods--such as LocalNewton, LTDA, and FedSophia--have incorporated second-order optimization in FL by performing iterative local updates on clients and applying \textit{simple mixing} of local parameters on the server. However, these methods often suffer from drift in local preconditioners, which significantly disrupts the convergence of parameter training, particularly in heterogeneous data settings. To overcome this issue, we refine the update rules by decomposing the ideal second-order update--computed using globally preconditioned global gradients--into parameter mixing on the server and local parameter updates on clients. As a result, our FedPM introduces \textit{preconditioned mixing} of local parameters on the server, effectively mitigating drift in local preconditioners. 
We provide a theoretical convergence analysis demonstrating a superlinear rate for strongly convex objectives in scenarios involving a single local update.
To demonstrate the practical benefits of FedPM, we conducted extensive experiments. 
The results showed significant improvements with FedPM in the test accuracy compared to conventional methods incorporating simple mixing, fully leveraging the potential of second-order optimization.
\end{abstract}

\begin{links}
    \link{Code}{https://github.com/rioyokotalab/fedpm}
    % \link{Extended version}{https://}
\end{links}

\section{Introduction}
Federated Learning (FL) is a distributed learning paradigm in which multiple clients collaboratively train a global parameter without sharing their local datasets \cite{kairouz2021advances, konevcny2016federatedb, konevcny2016federateda, mcmahan2017communication}, thereby offering benefits in privacy \cite{bonawitz2017practical,geyer2017differentially,ozdayi2021defending}, scalability \cite{bonawitz2019towards,li2020federated,oldenhof2023industry}, and communication efficiency \cite{lee2023layer,zang2024efficient}.

As summarized in Table \ref{tab:fl_methods_comparison}, FL methods can be broadly classified into four categories--FOGM, FOPM, SOGM, and SOPM--based on the combination of two key factors: (i) whether the update rule employs First-Order (FO) or Second-Order (SO) optimization methods, and (ii) whether the server-side processing relies on local Gradient Mixing (GM) or local Parameter Mixing (PM). FO methods primarily rely on the gradient--the first derivative of the loss function--to guide parameter updates. In contrast, SO methods also leverage the Hessian (the second derivative), which captures the curvature of the loss landscape. By utilizing this curvature information, SO methods can make more informed updates, which often accelerates the training process. Well-known Parallel SGD (PSGD;  also referred to as Minibatch SGD or Distributed SGD) \cite{collins2022fedavg,woodworth2020local}, which updates the global parameter through simple mixing of local gradients, falls into FOGM. In contrast, FedAvg \cite{mcmahan2017communication} is classified as FOPM, as the global parameter is obtained by averaging local updated parameters, where each client may perform multiple local updates before communication. Second-order counterparts--SOGM and SOPM--have also been explored. Notably, SOPM methods such as LocalNewton \cite{gupta2021localnewton}, Linear-Time Diagonal Approximation (LTDA) \cite{sen2023federated}, and FedSophia \cite{elbakary2024fed} aim to potentially accelerate and improve the efficiency of training by leveraging both (i) curvature information (i.e., Hessian) and (ii) multiple local updates on each client, thereby reducing communication overhead.

\begin{table*}[t]
\centering
\setlength{\tabcolsep}{1mm}
{\fontsize{9}{\baselineskip}\selectfont
\begin{tabular}{@{}llllccc@{}}
\toprule
\shortstack[l]{\textbf{Method} \\ \textbf{category}}
& \shortstack[l]{ \textbf{Representative methods} \\ \textbf{ } }
& \shortstack[l]{\textbf{Update} \\ \textbf{rules} }
& \shortstack[l]{\textbf{Server-side processing} \\ \textbf{ }}
& \shortstack[l]{\textbf{Global} \\ \textbf{FO opt.}}
& \shortstack[l]{\textbf{Global} \\ \textbf{SO opt.}}
& \shortstack[l]{\textbf{Multiple} \\ \textbf{local updates}}
\\
\midrule
\shortstack[l]{FOGM \\ {}} & \shortstack[l]{PSGD, Minibatch SGD \\ {} }
& \shortstack[l]{Eq. (\ref{eq:PSGD}) \\ {}}
& \shortstack[l]{Global parameter update using
\\ simple mixing of local gradients}
&\shortstack[l]{\checkmark \\ {}}& \shortstack[l]{-- \\ {}} & \shortstack[l]{-- \\ {}}
\\
\midrule
FOPM & FedAvg, FedProx, SCAFFOLD 
& Eq. (\ref{eq:FedAvg})
& Simple mixing of local parameters
&\checkmark &-- & \checkmark 
\\
\midrule
\shortstack[l]{
SOGM 
\\ {}}
& 
\shortstack[l]{
FedNL, FedNew, FedNS
\\ {}}
& \shortstack[l]{
Eq. (\ref{eq:SOGM}) \\ {}}
& \shortstack[l]{
Global parameter update using simple mixing\\
 of local gradients and preconditioning matrices
}
&\shortstack[l]{-- \\ {}} 
& \shortstack[l]{\checkmark \\ {}} 
& \shortstack[l]{-- \\ {}} 
\\
\midrule
SOPM
& LocalNewton, LTDA, FedSophia
& Eq. (\ref{eq:LocalNewton})
& Simple mixing of local parameters
& -- & --$^{(\ast 1)}$ & \checkmark
\\
\cline{2-7}
& 
\textbf{FedPM (ours)} 
&
Eq. (\ref{eq:fedpm_multi}) 
& 
\textbf{Preconditioned mixing} of local parameters
& -- & \checkmark$^{(\ast 2)}$ & \checkmark 
\\
\bottomrule
\end{tabular}}
\caption{
Comparison of representative FL methods. All methods can be classified into four categories based on the combination of two key factors: (i) optimization manner: either First-Order (FO) or Second-Order (SO) methods, and (ii) server-side processing: either Gradient Mixing (GM) or Parameter Mixing (PM). Global FO/SO opt. correspond to update rules (\ref{eq:GlobalUpdateFO}) and (\ref{eq:GlobalUpdateSO}), respectively. While details are explained in Sec. \ref{sec:related_works}, we reveal that existing SOPM methods cannot achieve global SO optimization due to the simple mixing of local parameters ($^{\ast 1}$). In contrast, our proposed method, FedPM, employs \textit{preconditioned mixing} of local parameters on the server, which enables global SO optimization with single local update setting ($^{\ast 2}$). }
\label{tab:fl_methods_comparison}
\end{table*}

In conventional SOPM methods, server-side processing typically consists of simple mixing of local parameters transmitted from local clients. However, this approach suffers from a fundamental limitation: it does not correspond to a true global second-order optimization. As detailed in Sec. \ref{sec:SOworks}, our preliminary analysis revealed that such simple mixing results in a global parameter update based on the sum of locally preconditioned local gradients, rather than a globally preconditioned global gradient. This mismatch introduces a critical bottleneck in federated learning, particularly under heterogeneous datasets. In heterogeneous data settings, the local curvature--captured by each client's Hessian--often poorly approximates the global curvature. 
Consequently, this reliance on mismatched local curvature information can lead to a drift in local preconditioners, which ultimately hinders convergence and prevents the full potential of second-order optimization from being realized, particularly in highly heterogeneous data settings.

%This discrepancy would disrupt convergence and  the parameter from fully leveraging the power of second-order information.

%this fails to adequately capture the curvature of the global data distribution, it would not work well under heterogeneous datasets across clients. 

%In typical FL scenarios, data heterogeneous is assumed. Thus, 

%data distributions across clients, the local data curvature can drastically differ from the global curvature. This discrepancy can disrupt convergence, preventing the model from fully leveraging the power of second-order information.

%are attractive for their reduced communication frequency, they suffer from a significant drawback. As shown in Table \ref{tab:fl_methods_comparison}, 

%prevents the global update from being equivalent to the ideal single-node update, even for the base case of a single local update. 

To enable global parameter updates that fully leverage second-order information, we propose a novel SOPM method termed Federated Preconditioned Mixing (FedPM). As outlined above, conventional SOPM methods rely on simple mixing of local parameters on the server, yielding global parameter updates using locally preconditioned local gradients. In contrast, our approach derives client-server update rules by decomposing an ideal global second-order optimization using globally preconditioned global gradients. This formulation leads to a core innovation: replacing simple mixing on the server with \textit{preconditioned mixing} of local parameters--a curvature-aware local parameter mixing. By incorporating this preconditioned mixing, FedPM effectively captures global curvature information, leading to improved convergence even in heterogeneous data settings. As highlighted in Table \ref{tab:fl_methods_comparison}, FedPM is the only SOPM method that achieves global second-order optimization while allowing multiple local updates. Our contributions are summarised as follows:\\
%
%This core modification is derived by reformulating the ideal single-node second-order update into a federated structure. As a result, the composed global update in FedPM effectively uses a \textit{globally preconditioned global gradient}, aligning the optimization trajectory with the true global loss landscape even under data heterogeneity. Our contributions are as follows:\\
%
%that employs a global update using globally preconditioned global gradients, thereby . 
%
%To mitigate this issue, we propose , which is categorized into SOPM, 
%a novel  method that, unlike its predecessors, 
%achieves equivalence with the ideal single-node update while still allowing for multiple local updates extension. 
%To overcome this limitation, we propose Federated Preconditioned Mixing (FedPM), a novel SOPM method that, unlike its predecessors, achieves equivalence with the ideal single-node update while still allowing for multiple local updates extension. As highlighted in Table \ref{tab:fl_methods_comparison}, FedPM is the only SOPM method that satisfies all desirable properties. It replaces simple parameter averaging on the server with a \textit{preconditioned mixing} step. This core modification is derived by reformulating the ideal single-node second-order update into a federated structure. As a result, the composed global update in FedPM effectively uses a \textit{globally preconditioned global gradient}, aligning the optimization trajectory with the true global loss landscape even under data heterogeneity. Our contributions are as follows:\\
\textbf{Derivation of curvature-aware SOPM method (FedPM)}: We derive FedPM, a novel FL algorithm featuring preconditioned mixing on the server, which aligns local updates with the global curvature, expecting robustness towards heterogeneous data settings (Sec. \ref{sec:fedpm}).\\
\textbf{Convergence analysis}: A theoretical convergence rate of FedPM is given for a limited condition (strongly convex objective with single local update). It demonstrates a superlinear convergence rate (Sec. \ref{subsec:analysis}).
%consistent with ideal (but communication-intensive) second-order methods like FedNL (Sec. \ref{subsec:analysis}).
\\
\textbf{Practical implementation}: To train large-scale models (e.g., non-convex deep learning models), an efficient preconditioner approximation method, FOOF \cite{benzing2022gradient}, is employed (Sec. \ref{subsec:approx}).\\
\textbf{Empirical superiority}: We empirically demonstrated that FedPM significantly outperformed existing FO and SO methods in convergence speed and best test accuracy, especially on heterogeneous data settings (Sec. \ref{sec4}).

\section{Related works}
\label{sec:related_works}

\subsection{First-Order (FO) distributed optimization}
We categorize FO distributed optimization methods into two categories: FOGM and FOPM. 

FOGM methods, including Parallel SGD (PSGD) \cite{collins2022fedavg,woodworth2020local}, update the global parameter $\vtheta \in \mathbb{R}^d$ by aggregating local gradients from each client in every communication round, as

\begin{align}
&\text{Client: } \vg_i^{(t)} = \nabla f_{i}(\vtheta^{(t)}) \quad \quad \quad \quad \forall i \in \{1,\ldots,N\},
\label{eq:PSGD} \\
&\text{Server: } \vtheta^{(t+1)} = \vtheta^{(t)} - \frac{\eta}{N} \sum_{i=1}^N \vg_i^{(t)},
\notag
\end{align}
where the initial parameter $\vtheta^{(0)}$ is given, $f_{i}$ is a twice-differentiable local loss function, $\nabla f_{i}(\vtheta^{(t)})$ represents the local gradient computed using local datasets, $N$ is the number of clients, $\eta$ is the learning rate, and $t \in \{ 0,\ldots, T-1 \}$ denotes the communication round index.
This simple method, however, incurs high communication costs. %due to frequent communication.

To improve communication-efficiency, FOPM methods like FedAvg \cite{mcmahan2017communication}, FedProx \cite{li2020federated}, and SCAFFOLD \cite{karimireddy2020scaffold} allow clients to perform multiple local updates ($K \geq 1$) and only transmit the resulting parameters:
\begin{align}
&\text{Client: } 
\vtheta_i^{(t,0)} = \vtheta^{(t)}, 
\hspace{46pt} \forall i \in \{1,\ldots,N\},
\label{eq:FedAvg}  \\
&\vtheta_i^{(t,k+1)} = \vtheta_i^{(t,k)} - \eta \nabla f_{i}(\vtheta_{i}^{(t,k)}), \hspace{5pt} \forall k \in \{0,\ldots,K-1\},
\notag
\\
&\text{Server: } \vtheta^{(t+1)} = \underbrace{\frac{1}{N}\sum_{i=1}^N \vtheta_i^{(t,K)}}_{\textrm{simple mixing}}.
% \hspace{20pt}(\textrm{simple mixing})
\notag
\end{align}

In the case of a single local update ($K=1$), the client-server update rules in FOGM and FOPM are combined, resulting in an equivalent global update for both FOGM and FOPM, as follows:
\begin{align}
\vtheta^{(t+1)} = \vtheta^{(t)} - \eta \Biggl( \underbrace{\frac{1}{N} \sum_{i=1}^N  
 \nabla f_{i}(\vtheta^{(t)})}_{\text{global gradient}} \Biggr).
\label{eq:GlobalUpdateFO} 
\end{align}

Equation (\ref{eq:GlobalUpdateFO}) represents an ideal FO optimization, where the global parameter is updated using global gradient, which is feasible when all datasets are centrally aggregated.

\subsection{Second-Order (SO) distributed optimization}
\label{sec:SOworks}

Similar to FO distributed optimization, SO distributed optimization methods can be categorized into SOGM and SOPM.

SOGM methods including FedNL \cite{safaryan2021fednl}, FedNew \cite{elgabli2022fednew}, and FedNS \cite{li2024fedns} are communication-intensive, requiring clients to transmit local gradients and Hessians ($\mP_i^{(t)}$) every round:
\begin{align}
&\text{Client: } \vg_i^{(t)} = \nabla f_{i}(\vtheta^{(t)}), \hspace{10pt}
\mP_i^{(t)}= \nabla^{2} f_{i}(\vtheta^{(t)}), \label{eq:SOGM}\\
&\hspace{150pt} \forall i \in \{1,\ldots,N\}, \notag\\
&\text{Server: } \vtheta^{(t+1)} = \vtheta^{(t)} - \eta \left(\frac{1}{N} \sum_{i=1}^N \mP_i^{(t)}\right)^{\hspace{-5pt}-1}\hspace{-7pt} \left( \frac{1}{N} \sum_{i=1}^N \vg_i^{(t)} \right). \notag
\end{align}

The communication-efficient alternative, SOPM, includes methods such as LocalNewton \cite{gupta2021localnewton}, LTDA \cite{sen2023federated}, and FedSophia \cite{elbakary2024fed}, where each client performs local SO updates, while the server simply mixes the local parameters, as
\begin{align}
&\text{Client: } 
\vtheta_i^{(t,0)} = \vtheta^{(t)}, 
\label{eq:LocalNewton}
\\
&\hspace{28pt} \vtheta_i^{(t,k+1)} = \vtheta_i^{(t,k)} \hspace{-5pt}- \eta \left( \nabla^{2} f_{i}(\vtheta_{i}^{(t,k)}) \right)^{\hspace{-3pt}-1} \hspace{-5pt} \nabla f_{i}(\vtheta_{i}^{(t,k)}), \notag
\\
&\hspace{70pt} \forall i \in \{1,\ldots,N\}, \forall k \in \{0,\cdots,K-1\},
\notag
\\
&\text{Server: } \vtheta^{(t+1)} = \underbrace{ \frac{1}{N}\sum_{i=1}^N \vtheta_i^{(t,K)}}_{\textrm{simple mixing}}. %\hspace{30pt}(\textrm{simple mixing})
\notag
\end{align}

As shown in the following analysis, SOGM and an instance (\ref{eq:LocalNewton}) of SOPM are not equivalent, in contrast to the equivalence observed in their FO counterparts. By combining client-server update rules in (\ref{eq:SOGM}) under a single-local update setting ($K=1$), the SOGM updates simplifies to
\begin{align}
\fontsize{8.9}{\baselineskip}\selectfont
\vtheta^{(t+1)} \hspace{-2pt} = \vtheta^{(t)} \hspace{-2pt} - \eta \underbrace{ \Biggl( \frac{1}{N} \sum_{i=1}^N \nabla^2 f_i(\vtheta^{(t)}) \Biggr)^{\hspace{-3pt}-1} }_{\text{global preconditioner}}
\hspace{-2pt}\Biggl( \underbrace{\frac{1}{N} \sum_{i=1}^N \nabla f_{i}(\vtheta^{(t)})}_{\text{global gradient}} \Biggr).
\label{eq:GlobalUpdateSO} 
\end{align}
Eq. (\ref{eq:GlobalUpdateSO}) is an ideal SO optimization, where the global parameter is updated using a globally preconditioned global gradient, which is feasible when all datasets are centrally aggregated. In contrast, the SOPM update in (\ref{eq:LocalNewton}) under $K=1$ can be simplified as follows:
\begin{equation}
\vtheta^{(t+1)} = \vtheta^{(t)} - \frac{\eta}{N} \sum_{i=1}^N \underbrace{\left(\nabla^2 f_i(\vtheta^{(t)})\right)^{-1}}_{\text{local preconditioner}} \underbrace{\nabla f_{i}(\vtheta^{(t)})}_{\text{local gradient}}.
\label{eq:SOPM_K=1}
\end{equation}
In (\ref{eq:SOPM_K=1}), the global parameter is updated using the average of locally preconditioned local gradients. This differs from an ideal SO optimization in (\ref{eq:GlobalUpdateSO}), which employs a globally preconditioned global gradient. 
 
We showed that conventional SOPM (\ref{eq:LocalNewton}) captures local curvatures instead of global curvature. This mismatch becomes problematic in FL, where data heterogeneity causes local preconditioners to be poor estimates of the global one, leading to suboptimal convergence. 

\section{Proposed method}
\label{sec3}

To capture global curvature--rather than local curvatures--for faster and more stable global parameter optimization, we propose FedPM, a novel SOPM method. 
%
%, method designed to resolve the key limitation of its predecessors: the failure to align with the ideal single-node second-order update, a gap highlighted in the previous section and Table \ref{tab:fl_methods_comparison}. 
%
Our goal is to design an SOPM method in which the combination of client-server update rules yields the update in (\ref{eq:GlobalUpdateSO}), which captures global curvature. To achieve this, our FedPM is derived by decomposing (\ref{eq:GlobalUpdateSO}) into client-server update rules, as detailed in Sec. \ref{sec:fedpm}. A theoretical convergence analysis of FedPM under a certain condition is presented in Sec. \ref{subsec:analysis}. 
To address computational overheads, an efficient preconditioner approximation technique is introduced in Sec. \ref{subsec:approx}, enabling the use of large-scale models (e.g., Deep Neural Networks (DNNs)) as evaluated in Sec.~\ref{sec4}. Further investigation regarding differences between our FedPM and related works is highlighted in Appendix \ref{secap:distinction}.

\subsection{Derivation of FedPM}
\label{sec:fedpm}
To address the issues of conventional SOPM methods discussed in Sec. \ref{sec:SOworks}, our core idea is to reformulate the global SO optimization in (\ref{eq:GlobalUpdateSO}) into local parameter updates on clients and parameter mixing on the server. For this goal, (\ref{eq:GlobalUpdateSO}) is reformulated as follows:
\begin{align}
&\vtheta^{(t+1)}=\vtheta^{(t)}-\eta (\mP^{(t)} )^{-1} \vg^{(t)}
\label{eq:fedpm_derivation} \\
&\mbox{\fontsize{9.5}{\baselineskip}\selectfont $=(\mP^{(t)})^{-1} \mP^{(t)}\vtheta^{(t)}-\eta (\mP^{(t)} )^{-1}\displaystyle\frac{1}{N}\sum_{i=1}^N\mP_i^{(t)} (\mP_i^{(t)})^{-1}\vg_i^{(t)}$}\notag \\
&\mbox{\fontsize{9}{\baselineskip}\selectfont $=\displaystyle\frac{1}{N}\sum_{i=1}^N\left((\mP^{(t)})^{-1} \mP_i^{(t)}\vtheta^{(t)}-\eta (\mP^{(t)} )^{-1}\mP_i^{(t)} (\mP_i^{(t)})^{-1}\vg_i^{(t)}\right)$}\notag \\
&=\frac{1}{N}\sum_{i=1}^N (\mP^{(t)})^{-1}\mP_i^{(t)}\left(\vtheta^{(t)}-\eta (\mP_i^{(t)})^{-1}\vg_i^{(t)}\right), \notag
\end{align}
where $\vg_i^{(t)}=\nabla f_{i}(\vtheta^{(t)}), \vg^{(t)}=\frac{1}{N}\sum_{i=1}^N\vg_i^{(t)},\mP_i^{(t)}=\nabla^{2} f_{i}(\vtheta^{(t)}),\mP^{(t)}=\frac{1}{N}\sum_{i=1}^N\mP_i^{(t)}$.
A straightforward decomposition of this leads to the update rules of our FedPM, as follows:
\begin{align}
&\textbf{[FedPM with single local update $(K=1)$]} 
\notag \\
&\text{Client: } 
\vtheta_i^{(t)} = \vtheta^{(t)}, 
\mP_i^{(t)} = \nabla^2 f_i(\vtheta^{(t)}), 
\hspace{5pt} \forall i \in \{1,\ldots,N\}, \label{eq:fedpm_single}\\
&\hspace{28pt} \vtheta_i^{(t+1)} = \vtheta_{i}^{(t)} - \eta (\mP_i^{(t)})^{-1} \nabla f_i(\vtheta^{(t)}), \notag\\
&\text{Server: } \mP^{(t)} = \frac{1}{N} \sum_{i=1}^N \mP_i^{(t)}, \notag\\
&\hspace{27pt} \vtheta^{(t+1)} = \underbrace{\frac{1}{N} \sum_{i=1}^N (\mP^{(t)})^{-1} \mP_i^{(t)} \vtheta_i^{(t+1)}}_\text{preconditioned mixing}, \notag
\end{align}
where local parameter mixing on the server is referred to as \textit{preconditioned mixing}. A natural extension of (\ref{eq:fedpm_single}), allowing multiple local parameter updates, results in the following:
\begin{align}
&\textbf{[FedPM with multiple local updates $(K > 1)$]} 
\notag \\
&\text{Client: } 
\vtheta_i^{(t,0)} = \vtheta^{(t)}, 
\hspace{50pt} \forall i \in \{1,\ldots,N\}, \label{eq:fedpm_multi}\\
& \hspace{26pt} \mP_i^{(t,k)} = \nabla^2 f_i(\vtheta_{i}^{(t,k)}), 
\hspace{15pt}
\forall k \in \{0,\cdots,K-1\},
\notag \\
&\hspace{26pt} \vtheta_i^{(t, k+1)} = \vtheta_{i}^{(t, k)} - \eta ( {\mP_i^{(t, k)}} )^{-1} \nabla f_i(\vtheta^{(t, k)}), 
\notag\\
% \end{align}
% \noindent
% \begin{align}
&\text{Server: } \mP^{(t)} = \frac{1}{N} \sum_{i=1}^N \mP_i^{(t, K-1)}, \notag\\
&\hspace{30pt} \vtheta^{(t+1)} = \underbrace{\frac{1}{N} \sum_{i=1}^N (\mP^{(t)})^{-1} \mP_i^{(t,K-1)} \vtheta_i^{(t,K)}}_\text{preconditioned mixing}.\notag
\end{align}

The key feature of FedPM lies in its use of preconditioned mixing, which enables effective capture of global curvature information. However, unlike conventional SOPM methods (e.g., LocalNewton, LTDA, and FedSophia), FedPM incurs additional communication costs as it requires transmitting not only local parameters but also local preconditioners. Nevertheless, even with this additional communication overhead, FedPM is preferable in heterogeneous data settings. By capturing global curvature--rather than relying on statistically biased local curvatures--FedPM can achieve faster and more stable convergence even under heterogeneous data distributions.

%However, this trade-off is justified by a significant advantage: our derivation ensures that the global update for FedPM with $K=1$ is mathematically equivalent to the ideal single-node second-order update (\ref{eq:GlobalUpdateSO}). 
%This allows FedPM to satisfy the crucial 'Equivalence to Single-Node Update' property in Table \ref{tab:fl_methods_comparison}, directly resolving the problem described in Sec. \ref{sec:SOworks}. By doing so, FedPM ensures the global preconditioner is properly incorporated, overcoming the reliance on purely local curvature information that hampers other SOPM methods in heterogeneous settings. 

%As a result, FedPM is expected to be a more adaptable and robust approach in real-world FL scenarios. Moreover, by allowing multiple local updates, FedPM mitigates the high communication costs of SOGM methods, distinguishing itself as more than a simple extension of gradient-mixing approaches.

% In the next subsection, we provide a theoretical convergence analysis of FedPM in the single local parameter update setting (\ref{eq:fedpm_single}).

\subsection{Theoretical convergence analysis}\label{subsec:analysis}
We provide a convergence analysis of FedPM under a limited condition. In line with most convergence proofs for second-order methods \cite{safaryan2021fednl, elgabli2022fednew, li2024fedns}, our analysis assumes the loss function is strongly convex. Furthermore, the analysis is confined to the case of a single local update ($K=1$), as the drift in local preconditioners during multiple updates complicates a rigorous convergence analysis. The analysis also disregards stochastic noise by assuming full local gradients and Hessians are available.
% Our analysis builds upon the strategy used for FedNL, since FedPM with single local uprate ($K=1$) follows the same global parameter update rule given in (\ref{eq:GlobalUpdateSO}). A subtle but important distinction lies in the use of the most recent Hessian for preconditioning, which is a more direct application of Newton's method than the stale Hessian used in the original FedNL's proof in \cite{safaryan2021fednl}.
Due to space limitations, we present only informal statements here; formal statements and proofs are provided in Appendix \ref{secap:thproof}. 
\begin{theorem}[(Informal) Convergence rate of FedPM under single local update]
\label{thm:fedpm_informal}
Under the assumptions of strong convexity and Hessian smoothness, and given an initial parameter condition that implicitly assumes sufficiently close to the optimal solution, the FedPM algorithm with a single local update ($K=1$) as defined in (\ref{eq:fedpm_single}) achieves a superlinear convergence rate for the global parameter. 
\end{theorem}

\noindent
\textbf{Discussion of Theorem \ref{thm:fedpm_informal}}: 
% These results validate that FedPM, with only single local update per communication round, converges at a superlinear rate. Moreover, (\ref{eq:lyapnov_converge}) implies that local Hessians $\nabla^2 f_i(\vtheta^{(t)})$ at each client will approach the global solution $\nabla^2 f_i(\vtheta^*)$. However, Theorem \ref{thm:fedpm} presents several limitations. Primarily, it does not account for stochastic noise in gradients and Hessians, relying instead on full local Hessians and gradients computed from all local data samples. Our experiments in Sec. \ref{sec4} include two scenarios: (Test 1) described in Sec. \ref{subsec:test1}, adhering to the analysis's prerequisites by employing a strongly convex model with full local Hessians and gradients; and (Test 2) in Sec. \ref{subsec:test2}, which empirically examines the negative impact of stochastic noise on Hessians and gradients in non-convex DNN models. Another limitation of this analysis is that it applies solely to the scenario of single local update $(K=1)$. 
% When multiple local updates are allowed as in (\ref{eq:fedpm_multi}), mathematical equivalence between the global parameter update in FedPM and that in FedNL no longer holds. As a result, the theoretical guarantees provided by Theorem \ref{thm:fedpm} do not extend to these scenarios. The convergence analysis for FedPM with multiple local parameter updates remains a work for the future. 
This theoretical result demonstrates the potential of FedPM to achieve a superlinear convergence rate, providing a theoretical validation of FedPM's faster and more stable convergence. 
%Moreover, the analysis shows that the local Hessians also approach their optimal values. 
Our analysis builds upon the strategy used for FedNL, since FedPM with a single local update ($K=1$) follows the same global parameter update rule given in (\ref{eq:GlobalUpdateSO}). A subtle but important distinction lies in the use of the most recent Hessian for preconditioning, which is a more direct application of Newton's method than the stale Hessian used in the original FedNL's proof in \cite{safaryan2021fednl}.
% However, several limitations remain. For instance, in line with most convergence proofs for second-order methods, the loss function is assumed to be strongly convex, and the analysis disregards stochastic noise by assuming full local gradients and Hessians are available \cite{safaryan2021fednl}. Furthermore, the analysis is confined to the case of single local update ($K=1$), as the drift in local preconditioners during multiple updates complicates a rigorous convergence analysis.
Based on the constraints outlined above, our numerical experiments in Sec. \ref{sec4} are designed with two objectives: Test 1 in Sec. \ref{subsec:test1} aims to empirically validate Theorem \ref{thm:fedpm_informal}, using a strongly convex loss function with full gradient under single local update. In contrast, Test 2 in Sec. \ref{subsec:test2} focuses on more practical scenarios, evaluating FedPM on non-convex DNN model training. 

%Extending this convergence analysis to the more general case with multiple local updates is a direction for future work. 

%(i.e., without stochastic noise, a condition we match in our strongly convex experiment (Test 1, Sec. \ref{subsec:test1}) but relax in our DNN experiments (Test 2, Sec. \ref{subsec:test2}). 
%

\subsection{Preconditioner approximation method}\label{subsec:approx}

To apply FedPM to DNN models, where direct computation of the full Hessian matrix is impractical, we employ an efficient preconditioner approximation. A popular strategy is to leverage the Fisher Information Matrix (FIM) to approximate the Hessian \cite{amari1998natural,martens2020new}. Under common loss functions like cross-entropy, the FIM is equivalent to the Generalized Gauss-Newton (GGN) matrix, a widely used Hessian approximation in deep learning \cite{martens2020new, schraudolph2002fast}.
% This approach is compatible with FedPM because it explicitly constructs a matrix that can be transmitted to the server for preconditioned mixing.
\begin{table}[t]
  \centering
  \setlength{\tabcolsep}{1mm}
  {\fontsize{9}{\baselineskip}\selectfont
  % \begin{tabularx}{\columnwidth}{|>{\centering\arraybackslash}X|>{\centering\arraybackslash}X|>{\centering\arraybackslash}X|>{\centering\arraybackslash}X|}
  \begin{tabular}{|c|c|c|c|}
    \hline
    Method & Construction & Inverse & Communication \\ \hline
    FedPM & $\mathcal{O}(Md^2)$ & $\mathcal{O}(d^3)$ & $\mathcal{O}(d^2)$ \\ \hline
    FedPM w/ FOOF & $\mathcal{O}(Md)$ & $\mathcal{O}(d\sqrt{d/L})$ & $\mathcal{O}(d)$\\\hline
  \end{tabular}
  }
  \caption{Comparison of computation and communication costs between FedPM (with full Hessian) and FedPM with FOOF approximation.}
  \label{tab:comp_approx}
\end{table}

In this paper, we employ FOOF \cite{benzing2022gradient} as an FIM approximation method. FOOF is an effective FIM approximation method, particularly beneficial for DNNs holding a large number of parameters. It simplifies the FIM computation by approximating it with block-diagonal matrices, where each block corresponds to one of $L$ layers of the DNN. To further reduce computational complexity, each block is approximated as an uncentered covariance matrix of inputs of the layer. The update rule for FOOF in a specific layer $l$ of a DNN at the $i$-th client is given by:
\begin{equation}\label{eq:foof_local}
\vtheta_{i,l}^{(t,k+1)}=\vtheta_{i,l}^{(t,k)} - \eta  \operatorname{vec}\left(\left(\mA^{(t,k)}_{i,l}\right)^{-1} \operatorname{vec}^{-1} \left( \vg_{i,l}^{(t,k)} \right) \right),
\end{equation}
where $\vtheta_{i,l}^{(t,k)}$ and $\vg_{i,l}^{(t,k)}$ represent the parameter and the gradient of the $l$-th layer of $i$-th client at $k$-th local iteration during the $t$-th communication round, respectively, The matrix $\mA_{i,l}^{(t,k)}$ represents the uncentered covariance matrix of the inputs (FOOF matrix) of the $l$-th layer. The operators $\operatorname{vec}$ and $\operatorname{vec}^{-1}$ denote the operations to vectorize a matrix and revert a vector back to a matrix form, respectively.
FOOF is particularly useful in scenarios where balancing
computational cost and performance is crucial. The preconditioned mixing at the center server for the FOOF approximation is given by:
% {
% \fontsize{6.5}{2}\selectfont
\begin{align}\label{eq:foof_mixing}
\vtheta_l^{(t+1,0)} = \operatorname{vec}\Biggl(&\left(\frac{1}{N}\sum_{i=1}^N \mA_{i,l}^{(t,K-1)}\right)^{-1} \\
&\left(\frac{1}{N}\sum_{i=1}^N\mA_{i,l}^{(t,K-1)}\operatorname{vec}^{-1}\left(\vtheta_{i,l}^{(t,K)}\right)\right)\Biggr).\notag
\end{align}
% }

\begin{algorithm}[t]
\caption{Federated Preconditioned Mixing (FedPM)}
\label{alg:fedpm}
\begin{algorithmic}[1]
    \STATE {\bfseries Server Initialization:} Global parameter $\vtheta^{(0)}$
    \FOR{$t=0$ {\bfseries to} $T-1$}
        \STATE Server sends global parameter $\vtheta^{(t)}$ to all clients.
        \FOR{$i=1$ {\bfseries to} $N$ {\bfseries in parallel}}
            \STATE $\vtheta_i^{(t,0)} \leftarrow \vtheta^{(t)}$
            \FOR{$k=0$ {\bfseries to} $K-1$}
                \IF{using FOOF approximation}
                    \FOR{$l=1$ {\bfseries to} $L$}
                        \STATE Update $\vtheta_{i,l}^{(t,k+1)}$ per Eq. (\ref{eq:foof_local}).
                    \ENDFOR
                \ELSE
                    % \STATE $\vg_i^{(t,k)} \leftarrow \nabla f_i(\vtheta_i^{(t,k)})$.
                    \STATE $\mP_i^{(t,k)} \leftarrow \nabla^2 f_i(\vtheta_i^{(t,k)})$.
                    \STATE $\vtheta_i^{(t,k+1)} \leftarrow \vtheta_i^{(t,k)} - \eta (\mP_i^{(t,k)})^{-1} \nabla f_i(\vtheta_i^{(t,k)})$.
                \ENDIF
            \ENDFOR
            \STATE Client $i$ sends $\vtheta_i^{(t,K)}$ and preconditioner(s) (i.e., $\mP_i^{(t,K-1)}$ or $\{\mA_{i,l}^{(t,K-1)}\}_{l=1}^L$) to the server.
        \ENDFOR
        
        \STATE \textbf{Server aggregates:}
        \IF{using FOOF approximation}
            \FOR{$l=1$ {\bfseries to} $L$}
                 \STATE Update $\vtheta_{l}^{(t+1)}$ per Eq. (\ref{eq:foof_mixing}).
            \ENDFOR
        \ELSE
            \STATE $\mP^{(t)} \leftarrow \frac{1}{N}\sum_{i=1}^N \mP_{i}^{(t,K-1)}$.
            \STATE $\vtheta^{(t+1)} \leftarrow \frac{1}{N}\sum_{i=1}^N (\mP^{(t)})^{-1} \mP_{i}^{(t,K-1)} \vtheta_i^{(t,K)}$.
        \ENDIF
    \ENDFOR
\end{algorithmic}
\end{algorithm}

The computational and communication costs of the FedPM with the FOOF approximation, compared to FedPM without any preconditioner approximation, are summarized in Table \ref{tab:comp_approx}. In this comparison, the model architecture is limited to the $L$-layer fully connected neural network with each layer having consistent input and output dimensions of $\sqrt{d/L}$, and $M$ denotes the number of data samples used to compute the matrix. The table distinguishes the time complexities for constructing and inverting matrices, and the space complexity for storing and transmitting FOOF matrices.
Compared to the FedPM using full preconditioning matrices without approximation, the computational and communication costs are drastically reduced. The general procedure for FedPM, which accommodates both the full Hessian and its FOOF approximation, is detailed in Algorithm \ref{alg:fedpm}. The effectiveness of this approximation in FedPM is empirically evaluated in the experiments presented in Sec. \ref{subsec:test2}.

\section{Numerical experiments}
\label{sec4}

To empirically examine the performance of the proposed FedPM, we prepared two tests: (Test 1) a strongly convex model in Sec.~\ref{subsec:test1} and (Test 2) non-convex DNN models in Sec.~\ref{subsec:test2}. Additional experiments, including an ablation study on FedPM and a more realistic FL setting, are summarized in the Appendix \ref{secap:addexp}.

\subsection{Test 1 using a strongly convex model}
\label{subsec:test1}
\subsubsection{Experimental setups}
\noindent
\begin{figure}[t]
\centering
\includegraphics[width=\columnwidth]{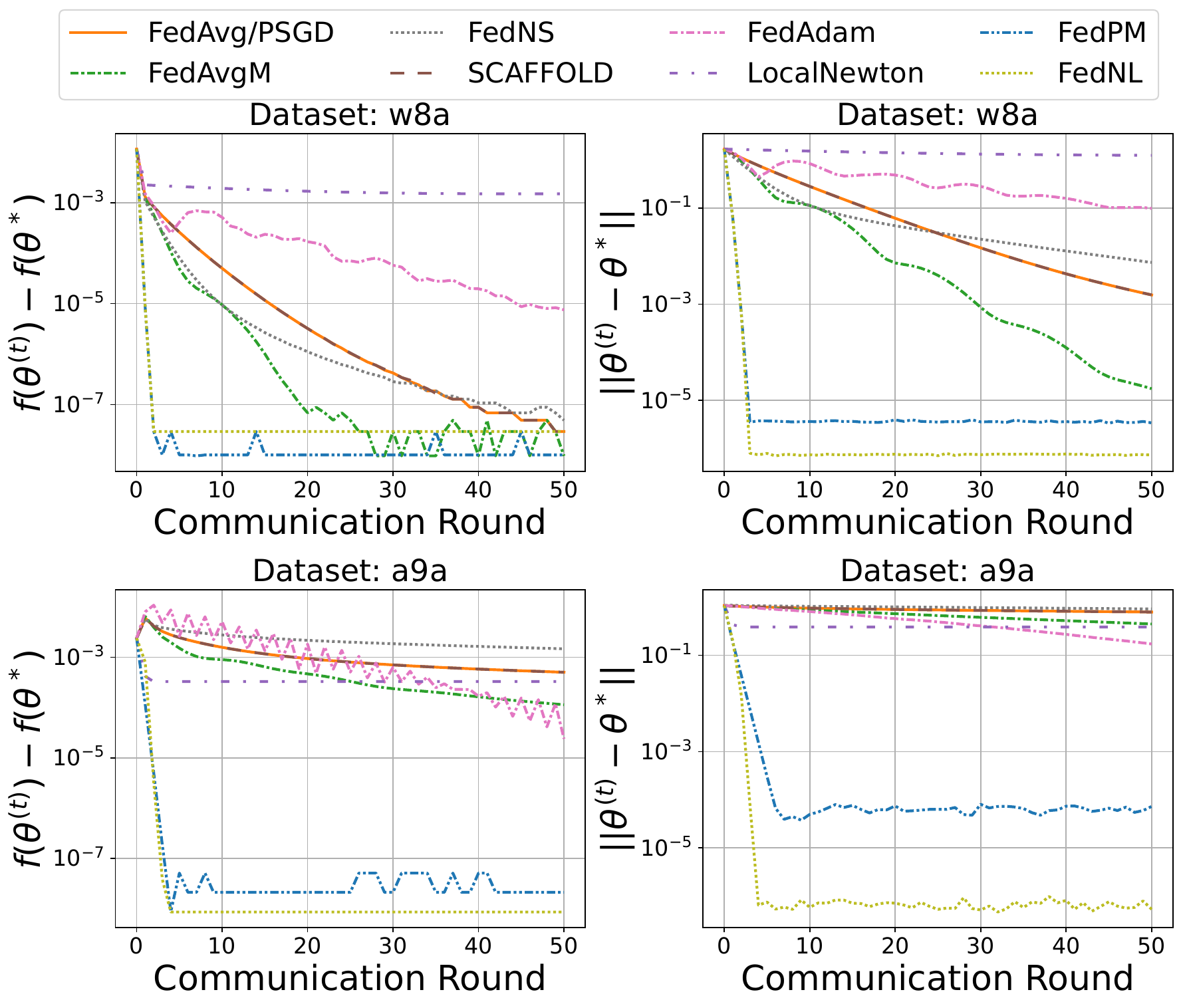}
\captionsetup{width=\columnwidth}
\caption{Convergence curves for Test 1 ($K=1$) using global parameter on w8a (top) and a9a (bottom). The left column displays the difference in function output at each round and at the optimal solution. The right column depicts the L2 norm of the difference between the parameter at each round and the optimal solution.}
\label{fig:test1}
\end{figure}

% To Appendix
% \begin{figure}[h!]
% \centering
% \includegraphics[width=0.8\columnwidth]{figures/comparison_w8a_a9a.pdf}
% \caption{Euclidean distance to the optimal solution ($||\vtheta^{(t)} - \vtheta^*||$) for Test 1, visually demonstrating the superlinear convergence of FedPM and FedNL.}
% \label{fig:superlinear_conv}
% \end{figure}

\begin{figure*}[t]
    \centering
    \includegraphics[width=0.75\paperwidth]{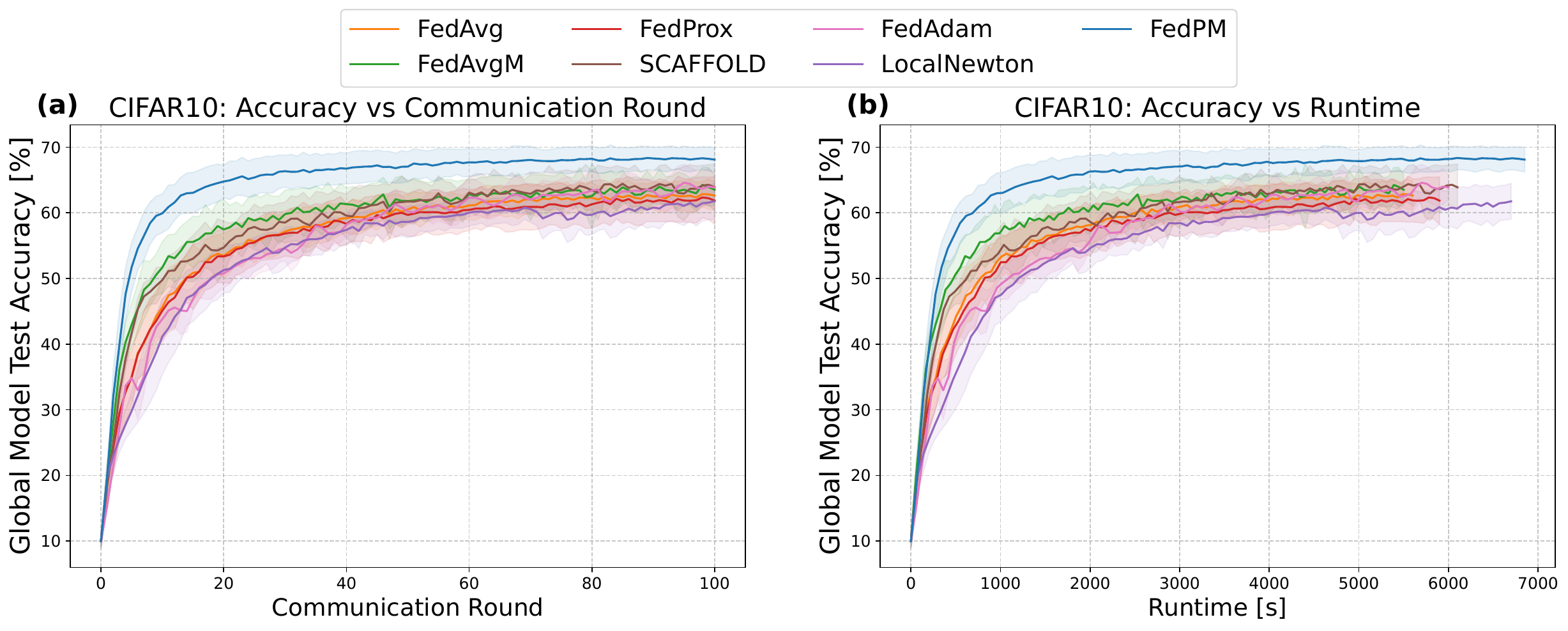}
    \caption{Convergence curves for Test 2 using test accuracy for the global parameter on CIFAR10 classification with heterogeneity level of $\alpha=0.1$ and 5 local epochs. (a) depicts test accuracy against communication rounds, whereas (b) shows test accuracy against runtime including communication overhead. The computation resource we used is summarized in Appendix \ref{secap:settings}. The shaded area depicts one standard deviation of results across three different random seeds.}
    \label{fig:curve}
    % \vspace{-10pt}
\end{figure*}

\noindent
\textbf{Datasets, model, and network configuration}:
As a toy experiment designed to satisfy the assumptions used in the theoretical analysis in Sec. \ref{subsec:analysis}, we employed a strongly convex model: logistic regression with squared $L_{2}$ regularization. The loss function is defined as $f(\vtheta) = \frac{1}{N} \sum_{i=1}^N f_i(\vtheta)$, where
% {\fontsize{8}{0}\selectfont
$f_i(\vtheta) = \frac{1}{M} \sum_{j=1}^{M} \log\left(1 + \exp(-y_{ij}\vx_{ij}^\top\vtheta)\right) + \frac{\lambda}{2}\|\vtheta\|^2$
% }
with $\{\vx_{ij},y_{ij}\}_{j\in[M]}$ representing data samples held by the $i$-th client. We used the w8a and a9a datasets from LibSVM \cite{chang2011libsvm}. The network configuration involves a central server with $N=142$ clients for w8a and $80$ clients for a9a communicating with single local update $(K=1)$ for $T=50$ communication rounds. The dataset was evenly and homogeneously distributed, with each client assigned $350$ data samples for w8a $(d=300)$ and $407$ data samples for a9a $(d=123)$. 

\noindent
\textbf{Comparison methods}:
We prepared seven comparison methods for evaluation, along with our proposed method:
FedAvg/PSGD \cite{collins2022fedavg,mcmahan2017communication}, 
FedAvgM \cite{hsu2019measuring},
SCAFFOLD \cite{karimireddy2020scaffold}, 
FedAdam \cite{reddi2021adaptive},
FedNS \cite{li2024fedns},
FedNL \cite{safaryan2021fednl},
LocalNewton \cite{gupta2021localnewton},
% FedNew \cite{elgabli2022fednew}
and proposed
FedPM.
We excluded SOPM methods that use diagonal approximations for the Hessian, such as LTDA and FedSophia, as they are designed for scenarios where full Hessian computation is intractable. Since computing the full Hessian is feasible in this experiment, using such approximations would be suboptimal.
In Test 1, stochastic gradient computation and FOOF-based preconditioner approximation were not applied to follow the convergence analysis in Sec. \ref{subsec:analysis}.

\noindent
\textbf{Hyperparameter settings}:
Hyperparameter tuning was conducted to find the best learning rate for each method. Details of the settings are described in Appendix \ref{secap:settings}.
\subsubsection{Experimental results}
Figure \ref{fig:test1} illustrates the performance of each method on the w8a and a9a datasets. The left column displays the difference between the function output at each round and the optimal value, $|f(\vtheta^{(t)})-f(\vtheta^*)|$, while the right column depicts the L2 norm of the difference between the parameter at each round and the optimal solution, $\|\vtheta^{(t)}-\vtheta^*\|$. The optimal solution $\vtheta^*$ was determined as the parameter obtained after $20$ iterations of standard Newton's method using all data samples. The initial parameter was drawn from a normal distribution centered on $\vtheta^*$ with a standard deviation of $0.1$. This choice ensures that the initial parameter is sufficiently close to the optimal solution, which is a necessary condition for the validity of the convergence analysis. In these experiments, both FedPM and FedNL significantly outperformed other methods, consistently achieving smaller function output gaps and parameter distances, thus indicating more efficient convergence. Notably, the rapid, accelerating decrease in the parameter distance for FedPM and FedNL shown in the right-hand plots is consistent with the superlinear convergence rate shown in Theorem \ref{thm:fedpm_informal}.
Although FedNL and FedPM with $K=1$ are equivalent in global parameter updates, minor discrepancies were observed in the outcomes. These differences may be attributed to precision errors resulting from variations in their respective update procedures.

\subsection{Test 2 using non-convex DNN models}
\label{subsec:test2}
\subsubsection{Experimental setups}

\noindent\\
\textbf{Datasets, models, and network configuration}:
To empirically assess the effectiveness of the proposed FedPM with non-convex DNN models, we conducted tests using two image classification benchmarks: (T1) CIFAR10\footnote{\url{https://www.cs.toronto.edu/~kriz/cifar.html}} with simple CNN as in \cite{li2021model, luo2021no}, and (T2) CIFAR100 with ResNet18 \cite{he2016deep}, where BatchNorm layers \cite{ioffe2015batch} were replaced by GroupNorm layers \cite{wu2018group} to enhance robustness against data heterogeneity. The center server and $N=10$ clients can communicate every $\{ 1, 5, 10 \}$ local update epochs across $T=100$ communication rounds. Client sampling was not employed in the main paper's experimental results, as its potential impacts have been explored in Appendix \ref{secap:addexp}. Data heterogeneity among clients was intentionally induced by varying the Dirichlet distribution concentration parameter $\alpha$ ($\alpha \in \{ 0.1, 1.0 \}$) following the implementation in \cite{vogels2021relaysum}, with a lower $\alpha$ indicating stronger data heterogeneity. Each client holds a different number of data samples, and the distribution of data samples across clients is illustrated in Appendix \ref{secap:settings}. We empirically investigated the impact of the number of local update epochs and the levels of data heterogeneity $\alpha$ on the test accuracy of global parameters. 

\noindent
\textbf{Comparison methods}:
In addition to the methods evaluated in Test 1, we added a comparison with FedProx \cite{li2020federated} in Test 2. However, SOGM methods like FedNL and FedNew were found to be computationally prohibitive for these large-scale DNN experiments, failing to complete a sufficient number of rounds within a practical time budget, and were thus excluded from this comparison. Additionally, whereas the local full Hessian was computed as for preconditioner in Test 1, it was approximated using FOOF for LocalNewton and FedPM in Test 2, as detailed in Sec. \ref{subsec:approx}. Statistical significance of FedPM's result against the comparison methods is summarized in Appendix \ref{secap:settings}. Additionally, empirical profiling results are provided in Appendix \ref{secap:profiling}.

\noindent
\textbf{Hyperparameter settings}:
The learning rate, the norm value for gradient clipping, the weight decay coefficient (regularization coefficient $\mu$ in FedProx), the damping term for second-order optimization methods, and unique hyperparameters for each method such as the momentum value in FedAvgM, were tuned to achieve the highest average of the highest test accuracy accross three different random seeds for each method. Additionally, we computed FOOF matrices only at the end of each round, just before the communication, to further improve efficiency for second-order methods. We empirically observed that this approach did not degrade performance compared to periodically updating the matrices during local training. For FOOF matrix computation, we utilized the full local dataset. An ablation study on the impact of the number of samples used for matrix computation is provided in Appendix \ref{secap:addexp}. Details of these configurations are summarized in Appendix \ref{secap:settings}.
\begin{table}[t]
  \centering
  \setlength{\tabcolsep}{1mm}
  % {\fontsize{9}{10.8}\selectfont
  \begin{tabular}{l||c|c|c|c}
    \hline
    \multirow{2}{*}{Method} & \multicolumn{2}{c|}{CIFAR10}        & \multicolumn{2}{c}{CIFAR100}        \\ \cline{2-5}
                            & $\alpha=1.0$     & $\alpha=0.1$     & $\alpha=1.0$     &   $\alpha=0.1$   \\ \hline    \hline
    FedAvg & 70.1$\pm$1.3 & 63.1$\pm$2.1 & 59.8$\pm$0.7 & 52.4$\pm$0.9             \\ \hline
    FedAvgM & 72.6$\pm$1.5 & 64.8$\pm$3.0 & 61.2$\pm$0.2 & 54.9$\pm$0.6               \\ \hline
    FedProx & 70.1$\pm$1.1 & 62.8$\pm$1.3 & 56.6$\pm$0.7 & 50.7$\pm$0.9               \\ \hline
    SCAFFOLD & 71.4$\pm$0.5 & 65.3$\pm$2.7 & 64.7$\pm$1.1 & 58.2$\pm$0.6     \\ \hline
    FedAdam & 70.4$\pm$1.6 & 64.6$\pm$2.8  & 58.3$\pm$0.3 & 51.9$\pm$0.4  \\ \hline
    LocalNewton & 73.4$\pm$1.4 & 62.4$\pm$2.3 & \textbf{70.2}$\pm$0.3 & 61.9$\pm$0.6  \\ \hline
    FedPM & \textbf{74.8}$\pm$0.5 & \textbf{68.6}$\pm$2.0 & 68.4$\pm$2.2 & \textbf{63.1}$\pm$0.8  \\ \hline
  \end{tabular}
  \caption{Average best test accuracy of global parameter (over three random seeds) for each method on various datasets with $\alpha=1.0$ or $\alpha=0.1$ (the smaller, the stronger heterogeneity), where local updates are performed for 5 epochs before each communication. The standard deviation across different seeds is also shown for each. The FOOF approximation was applied to both FedPM and LocalNewton.}  
  \label{tab:accuracy_table}
\end{table}
% \begin{table}[t]
%   \centering
%   \small
%   \begin{tabular}{|l|c|c|c|c|}
%     \hline
%     \multirow{2}{*}{Method} & \multicolumn{2}{c|}{CIFAR10}        & \multicolumn{2}{c|}{CIFAR100}        \\ \cline{2-5}
%                             & $\alpha=1.0$     & $\alpha=0.1$     & $\alpha=1.0$     &   $\alpha=0.1$   \\ \hline
%     FedAvg & 70.13$\pm$1.32 & 63.14$\pm$2.11 & 59.75$\pm$0.70 & 52.36$\pm$0.94             \\ \hline
%     FedAvgM & 72.58$\pm$1.48 & 64.75$\pm$2.99 & 61.16$\pm$0.18 & 54.88$\pm$0.55               \\ \hline
%     FedProx & 70.13$\pm$1.11 & 62.79$\pm$1.34 & 56.62$\pm$0.67 & 50.73$\pm$0.87               \\ \hline
%     SCAFFOLD & 71.37$\pm$0.52 & 65.26$\pm$2.74 & 64.70$\pm$1.06 & 58.21$\pm$0.63     \\ \hline
%     LocalNewton & 73.35$\pm$1.38 & 62.41$\pm$2.27 & \textbf{70.24}$\pm$0.27 & 61.93$\pm$0.61  \\ \hline
%     FedPM & \textbf{74.83}$\pm$0.54 & \textbf{68.62}$\pm$1.99 & 68.35$\pm$2.19 & \textbf{63.06}$\pm$0.79  \\ \hline
%   \end{tabular}
%   \caption{Average best test accuracy of global parameter (over three random seeds) for each method on various datasets with $\alpha=1.0$ or $\alpha=0.1$ (the smaller, the stronger heterogeneity), where local updates are performed for five epochs before each communication. The standard deviation across different seeds is also shown for each. The FOOF approximation was applied to both FedPM and LocalNewton.}  
%   \label{tab:accuracy_table}
% \end{table}
\subsubsection{Experimental results}
Table \ref{tab:accuracy_table} displays the highest test accuracy of a global parameter achieved by each method across four test settings, which combine two datasets/models with two levels of data heterogeneity.
These results demonstrate that our proposed FedPM outperformed all other comparison methods in test accuracy in most cases. Notably, FedPM exhibited superior performance at the stronger heterogeneity condition ($\alpha=0.1$), while LocalNewton's performance deteriorated significantly as data heterogeneity increased. These empirical findings suggest that preconditioned mixing improves alignment between local and global parameter updates, offering substantial benefits in heterogeneous data settings.

Figure \ref{fig:curve} illustrates convergence curves when local parameters are updated for five epochs before each communication under a stronger heterogeneous data setting ($\alpha=0.1$) on CIFAR10 classification. The results show that FedPM converged faster than other methods in terms of both communication rounds and runtime, thanks to its efficient preconditioner approximation method. This is crucial in FL because FedPM enables the model to reach the desired performance with fewer communication rounds. These findings confirm that FedPM’s ability to incorporate global curvature information through preconditioned mixing is useful for improving not only generalization and parameter accuracy but also the credibility of second-order FL in real-world scenarios. Convergence curves on CIFAR100 classification are included in Appendix \ref{secap:addexp}. 

\noindent
 \begin{figure}[t]
  \centering
  \includegraphics[width=0.9\columnwidth]{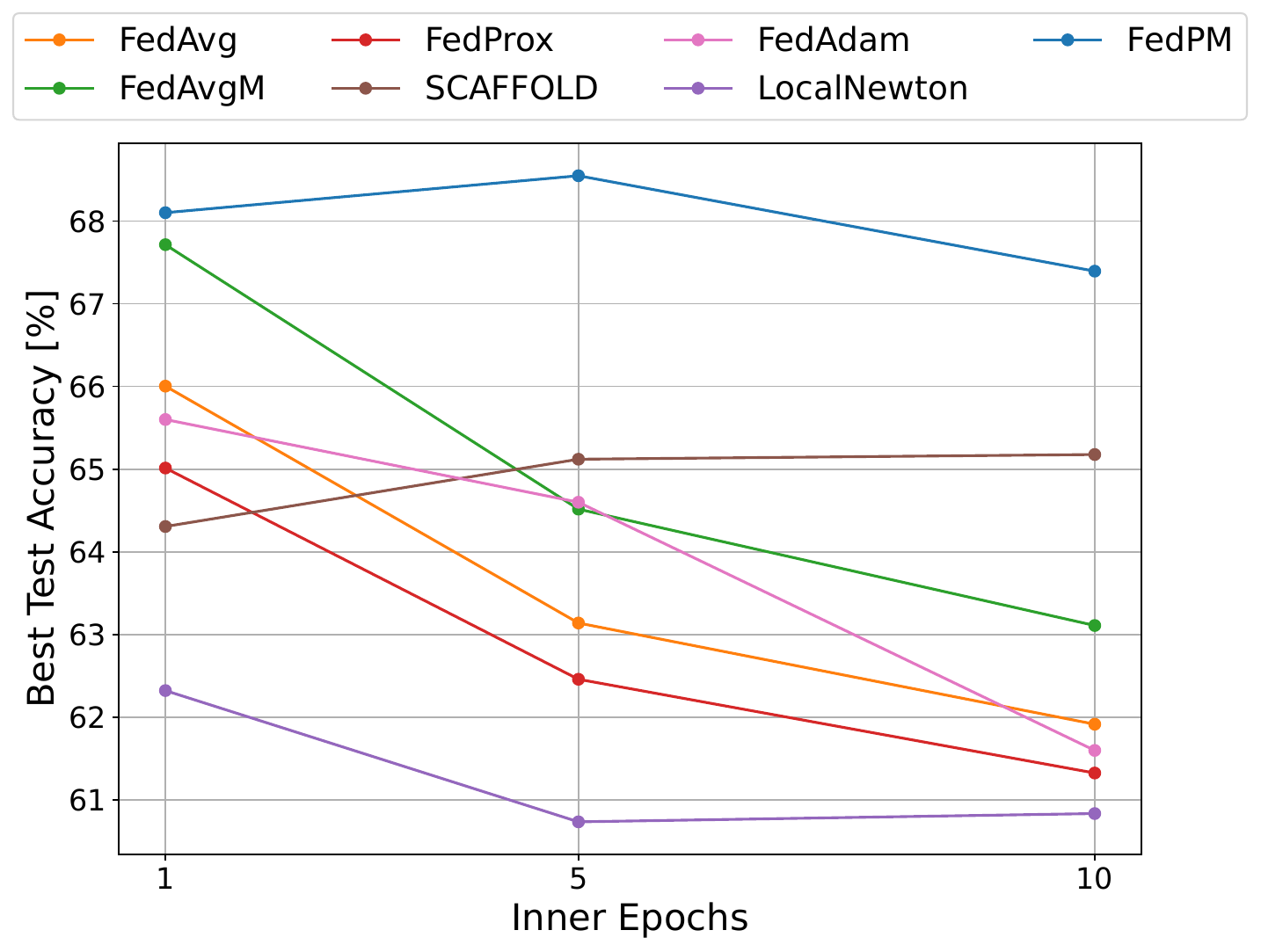}
  \captionsetup{width=\columnwidth}
  \caption{Relationship between the number of multiple local parameter updates and  
  test accuracy for CIFAR10 classification with $\alpha=0.1$. The plot shows the average highest test accuracy across three seeds.
  Experiments are conducted with a fixed total of $500$ communication rounds achieved by using 1 inner epoch (500 rounds), 5 inner epochs (100 rounds), and 10 inner epochs (50 rounds).}
  \label{fig:inner_epoch_result}
\end{figure}
Figure \ref{fig:inner_epoch_result} presents the CIFAR10 classification results with $\alpha = 0.1$ on varying number of inner epochs, where the total number of local epochs is fixed to $500$ by adjusting the number of communication rounds. Specifically, the three settings correspond to: $500$ rounds with $1$ local epoch per round, $100$ rounds with $5$ local epochs per round, and $50$ rounds with $10$ local epochs per round. For each setting, we report the average of the highest test accuracies achieved by the global model over three different random seeds. %The shaded areas in the figure indicate one standard deviation. 
The results indicate that FedPM consistently outperforms other methods, even when varying the number of local parameter updates.

\section{Conclusion}
We introduced FedPM, a novel Federated Learning (FL) method that unlocks the potential of second-order optimization. Unlike conventional Second-Order Parameter Mixing (SOPM) methods that are sensitive to data heterogeneity, FedPM incorporates preconditioned mixing at the server. This aligns parameter updates with the ideal global curvature, enhancing stability and convergence. Our theoretical analysis establishes a superlinear convergence rate for FedPM under strongly convex objectives and a single local update. Extensive experiments confirm its practical superiority: FedPM significantly outperformed state-of-the-art methods on both strongly convex and non-convex deep learning models, demonstrating particular robustness in highly heterogeneous settings. These results validate FedPM as a significant and practical advancement for second-order optimization in real-world federated scenarios.

\section*{Acknowledgements}
This work was supported by JST SPRING, Japan Grant Number JPMJSP2180.

\bibliography{main.bib}

\clearpage

\appendix
\section{Theoretical analysis (Theorem \ref{thm:fedpm_informal})}
\label{secap:thproof}
The formal statement of Theorem \ref{thm:fedpm_informal} in Section \ref{subsec:analysis} is given, including assumptions, conditions, and proofs. According to \cite{safaryan2021fednl}, the following assumptions are used:

%for the convergence rate of FedPM with a single local update, as discussed .

%\subsection{Assumptions and Conditions}

\begin{assumption}[Strongly convex]
\label{assumption1}
The average of local loss functions $f(\vtheta) = \frac{1}{N} \sum_{i=1}^{N} f_{i}(\vtheta)$ is $\mu$-strongly convex. 
\end{assumption}
\begin{assumption}[Hessian smoothness]
\label{assumption2}
Local loss functions $f_i(\vtheta)$ have Lipschitz continuous Hessians. Let $L_* (>0)$ and $L_F (>0)$ denote the Lipschitz constants with respect to two different matrix norms: spectral norm $\|\cdot\|$ and Frobenius norm $\|\cdot\|_{\rm F}$, respectively. Formally, we require the following conditions to hold for all $i \in \{ 1,\ldots,N \}$ and $\mathbf{a}, \mathbf{b} \in \R^d$:
\begin{align*}
\| \nabla^2 f_i(\mathbf{a}) - \nabla^2 f_i(\mathbf{b}) \| & \leq L_* \| \mathbf{a} - \mathbf{b} \|,\\
\| \nabla^2 f_i(\mathbf{a}) - \nabla^2 f_i(\mathbf{b}) \|_{\rm F} & \leq L_F \| \mathbf{a} - \mathbf{b} \|.
\end{align*}
\end{assumption}
\begin{condition}[Initial condition]
\label{condition1}
The initial global parameter $\vtheta^{(0)}$ and local Hessian $\nabla^2 f_i(\vtheta^{(0)})$ satisfy
\[\|\vtheta^{(0)} - \vtheta^*\| \leq \min\left\{\frac{\mu}{2\sqrt{2}L_F}_{\textstyle ,} \frac{\mu}{\sqrt{2} L_*}\right\},\]
\[\left\|\nabla^2 f_i(\vtheta^{(0)}) - \nabla^2 f_i(\vtheta^*)\right\|_{\rm F} \leq \frac{\mu}{2\sqrt{2}}, \]
where $\vtheta^*$ denotes the optimal solution of the global loss function $f$. 
\end{condition}

%\subsection{Convergence Theorem}

The formal statement of Theorem \ref{thm:fedpm_informal} is presented below:

\begin{formalthm}[(Formal) convergence rate of FedPM with single local update]
\label{thm:fedpm}
Under Assumptions \ref{assumption1}, \ref{assumption2}, and Condition \ref{condition1}, for the FedPM with single local update $(K=1)$ in (\ref{eq:fedpm_single}), we obtain the following linear rate for the Lyapunov function 
\begin{align*}
    \E\left[\Phi^{(t)}\right] & \leq \left(\frac{7}{12}\right)^t \Phi^{(0)},  %\label{eq:lyapnov_converge}
\end{align*}
where
\begin{align*}
&\Phi^{(t)}:=\frac{1}{n} \sum_{i=1}^n\left\|\nabla^2 f_i(\vtheta^{(t)})-\nabla^2 f_i\left(\vtheta^*\right)\right\|_{\rm F}^2\\
&\hspace{140pt}+6 L_{F}^2\left\|\vtheta^{(t)}-\vtheta^*\right\|^2.
\end{align*}
Additionally, the superlinear rate for the global parameter convergence is obtained:
\begin{align}    
    \E\left[\frac{\|\vtheta^{(t+1)} - \vtheta^*\|^2}{\|\vtheta^{(t)} - \vtheta^*\|^2}\right] & \leq \left(\frac{7}{12}\right)^t \left ( 2 + \frac{L_*^2}{12L_F^2} \right) \frac{\Phi^{(0)}}{\mu^2}.
\end{align}
\end{formalthm}

%\subsection{Proof of Theorem \ref{thm:fedpm}}
\begin{proof}
Proof of Theorem \ref{thm:fedpm} is presented. As explained in Sec. \ref{subsec:analysis}, the FedPM is originally derived from the second-order global parameter update rule, as in (\ref{eq:fedpm_derivation}) and (\ref{eq:fedpm_single}). Both FedPM with single local parameter update ($K=1$) and FedNL follow the same global parameter update rule. Our convergence analysis primarily follows the strategy in FedNL \cite{safaryan2021fednl}. The distinctions from the FedNL proof include disregarding the compression operation of the Hessian and maintaining a constant learning rate for the Hessian at $1$. First, we derive the recurrence relation for $\|\vtheta^{(t)}-\vtheta^*\|^2$:
\begin{align*}
&\left\| \vtheta^{(t+1)} - \vtheta^* \right\|^2 \hspace{-2pt}
\notag \\
&=   \left\|\vtheta^{(t)}-\vtheta^* - \left(\nabla^2 f(\vtheta^{(t)})\right)^{-1} \nabla f(\vtheta^{(t)}) \right\|^2 \notag \\
&\le \left\| \left(\nabla^2 f(\vtheta^{(t)})\right)^{-1} \right\|^2 \\ %\notag\\
%&\quad 
& \hspace{50pt}+ \left\|\nabla^2 f(\vtheta^{(t)})(\vtheta^{(t)}-\vtheta^*) - \nabla f(\vtheta^{(t)})\right\|^2 \notag \\
&\le \frac{2}{\mu^2}\left( \left\|\left(\nabla^2 f(\vtheta^{(t)}) - \nabla^2 f(\vtheta^*)\right)(\vtheta^{(t)}-\vtheta^*) \right\|^2 \right.\notag \\
&\qquad + \left.\left\|\nabla^2 f(\vtheta^*)(\vtheta^{(t)}-\vtheta^*) - \nabla f(\vtheta^{(t)}) + \nabla f(\vtheta^*) \right\|^2\right) \notag \\
&= \frac{2}{\mu^2}\left( \left\|\left(\nabla^2 f(\vtheta^{(t)}) - \nabla^2 f(\vtheta^*)\right)(\vtheta^{(t)}-\vtheta^*) \right\|^2 \right. \notag \\
&\qquad + \left. \left\| \nabla f(\vtheta^{(t)}) - \nabla f(\vtheta^*) - \nabla^2 f(\vtheta^*)(\vtheta^{(t)}-\vtheta^*) \right\|^2\right) \notag \\
&\le \frac{2}{\mu^2}\left(
\left\|\nabla^2 f(\vtheta^{(t)}) - \nabla^2 f(\vtheta^*)\right\|^2 \left\|\vtheta^{(t)}-\vtheta^* \right\|^2 \right. \notag \\
& \left. \hspace{140pt}
+ \frac{L_*^2}{4} \left\|\vtheta^{(t)}-\vtheta^* \right\|^4
\right) \notag \\
&=   \frac{2}{\mu^2} \left\|\vtheta^{(t)}-\vtheta^* \right\|^2 \left(
\left\|\nabla^2 f(\vtheta^{(t)}) - \nabla^2 f(\vtheta^*)\right\|^2 \right. \notag \\
& \hspace{140pt} \left. 
+ \frac{L_*^2}{4} \left\|\vtheta^{(t)}-\vtheta^* \right\|^2
\right) \notag \\
&\le \frac{2}{\mu^2} \left\|\vtheta^{(t)}-\vtheta^* \right\|^2 \left(
\left\|\nabla^2 f(\vtheta^{(t)}) - \nabla^2 f(\vtheta^*)\right\|^2_{\rm F} \right. \notag \\
& \hspace{140pt} \left. 
+ \frac{L_*^2}{4} \left\|\vtheta^{(t)}-\vtheta^* \right\|^2
\right)  \notag, 
\end{align*}
where we use $\nabla^2 f(\vtheta^{(t)}) \succeq \mu \mI$ in the second inequality following Assumption \ref{assumption1}. From the convexity of $\|\cdot \|^2_{\rm F}$, we have 
\begin{align*}
&\left\|\nabla^2 f(\vtheta^{(t)}) - \nabla^2 f(\vtheta^*) \right\|^2_{\rm F}\\
&= \left\| \frac{1}{n}\sum_{i=1}^n \left(  \nabla^2 f_i(\vtheta^{(t)}) - \nabla^2 f_i(\vtheta^*)  \right) \right\|^2_{\rm F} \\
&\leq \frac{1}{n}\sum_{i=1}^n \left\|\nabla^2 f_i(\vtheta^{(t)}) - \nabla^2 f_i(\vtheta^*) \right\|^2_{\rm F} := {\mathcal H}^{(t)}. 
\end{align*}
Integrating the above inequalities results in 
\begin{align}\label{eq:xk+1U}
&\left\|\vtheta^{(t+1)} - \vtheta^* \right\|^2 \leq \frac{2}{\mu^2} \left\|\vtheta^{(t)}-\vtheta^* \right\|^2 {\mathcal H}^{(t)}\\
& \hspace{140pt}+ \frac{L_*^2}{2\mu^2} \left\|\vtheta^{(t)}-\vtheta^* \right\|^4. \notag
\end{align}
When satisfying Condition \ref{condition1}; namely, $\left\|\vtheta^{(0)}-\vtheta^* \right\|^2 \leq \frac{\mu^2}{2L_*^2}$ and ${\mathcal H}^{(t)} \leq \frac{\mu^2}{8}$ for all $t\geq 0$, we show that $\left\|\vtheta^{(t)}-\vtheta^* \right\|^2 \leq \frac{\mu^2}{2L_*^2}$ for all $t\geq 0$ by induction. From $\left\|\vtheta^{(t)}-\vtheta^*\right\|^2 \leq \frac{\mu^2}{2L_*^2}$ for all $t \leq T$ and (\ref{eq:xk+1U}), we have 
\begin{align*}
&\left\|\vtheta^{(T+1)} - \vtheta^* \right\|^2 \\
&\leq \frac{1}{4} \left\|\vtheta^{(T)}-\vtheta^* \right\|^2 + \frac{1}{4} \left\|\vtheta^{(T)}-\vtheta^* \right\|^2 \leq \frac{\mu^2}{2L_*^2}. 
\end{align*} 
Thus, we have $\left\|\vtheta^{(t)}-\vtheta^*\right\|^2 \leq \frac{\mu^2}{2L_*^2}$ and ${\mathcal H}^{(t)} \leq \frac{\mu^2}{8}$ for $t\geq 0$. Using (\ref{eq:xk+1U}) again, we obtain 
\begin{equation}\label{eq:xk+1Ufix}
\left\|\vtheta^{(t+1)} - \vtheta^*\right\|^2 \leq \frac{1}{2} \left\|\vtheta^{(t)}-\vtheta^* \right\|^2. 
\end{equation}
From Assumption \ref{assumption2}, we have
$$
\mathbb{E} \left\|\nabla^2 f_i(\vtheta^{(t)}) - \nabla^2 f_i(\vtheta^*) \right\|^2_{\rm F} \leq L_F^2 \left\|\vtheta^{(t)}-\vtheta^* \right\|^2. 
$$
Then, we have 
\begin{align*}
\mathbb{E}[{\mathcal H}^{(t+1)}] &\leq L_F^2 \left\|\vtheta^{(t+1)}-\vtheta^* \right\|^2 \\
&\leq \frac{L_F^2}{2}\left\|\vtheta^{(t)}-\vtheta^* \right\|^2. 
\end{align*}
Using the above inequality and (\ref{eq:xk+1Ufix}), for the Lyapunov function
\begin{align*}
&\Phi^{(t)}=\frac{1}{n} \sum_{i=1}^n\left\|\nabla^2 f_i(\vtheta^{(t)})-\nabla^2 f_i\left(\vtheta^*\right)\right\|_{\rm F}^2\\
&\hspace{140pt}+6 L_{F}^2\left\|\vtheta^{(t)}-\vtheta^*\right\|^2,
\end{align*}
we deduce the following inequality.
\begin{align*}
\mathbb{E}\left[\Phi^{(t+1)}\right] & \leq \frac{L_F^2}{2} \left\|\vtheta^{(t)}-\vtheta^*\right\|^2 + 3L_F^2 \left\|\vtheta^{(t)}-\vtheta^*\right\|^2 \\ 
& \leq  \frac{7}{12} \left( {\mathcal H}^{(t)} + 6L_F^2 \left\|\vtheta^{(t)}-\vtheta^*\right\|^2 \right) \\ 
& = \frac{7}{12} \Phi^{(t)}. 
\end{align*}
Hence, $\mathbb{E}[\Phi^{(t)}] \leq \left(  \frac{7}{12}  \right)^t \Phi^{(0)}$. We further have $\mathbb{E}[{\mathcal H}^{(t)}] \leq \left(  \frac{7}{12}  \right)^t \Phi^{(0)}$ and $\mathbb{E}\left[ \left\|\vtheta^{(t)}-\vtheta^*\right\|^2\right] \leq \frac{1}{6L_F^2} \left(  \frac{7}{12}  \right)^t \Phi^{(0)}$ for $t\geq 0$. Assume $\vtheta^{(t)}\neq \vtheta^*$ for all $t$. Then from (\ref{eq:xk+1U}), we have 
$$
\frac{\|\vtheta^{(t+1)}-\vtheta^*\|^2}{\|\vtheta^{(t)}-\vtheta^*\|^2} \leq \frac{2}{\mu^2}{\mathcal H}^{(t)} + \frac{L_*^2}{2\mu^2}\left\|\vtheta^{(t)}-\vtheta^*\right\|^2, 
$$
and by taking expectation, we have 
\begin{align*}
&\mathbb{E} \left[  \frac{\|\vtheta^{(t+1)}-\vtheta^*\|^2}{\|\vtheta^{(t)}-\vtheta^*\|^2}  \right] \\
& \leq \frac{2}{\mu^2} \mathbb{E}[{\mathcal H}^{(t)}] + \frac{L_*^2}{2\mu^2} \mathbb{E}\left[\left\|\vtheta^{(t)}-\vtheta^*\right\|^2 \right] \\ 
& \leq  \left(  \frac{7}{12}  \right)^t \left(  2 + \frac{L_*^2}{12L_F^2}  \right) \frac{\Phi^{(0)}}{\mu^2}. 
\end{align*}
\end{proof}

\section{Further discussion on the novelty of FedPM}
\label{secap:distinction}

A potential misinterpretation of FedPM, particularly when multiple local updates ($K>1$) are involved, is to view it as a simple sequence of local second-order updates followed by a single global update step akin to SOGM methods. However, this view overlooks the unique role of \textit{preconditioned mixing}. To clarify the novelty of our approach, we explicitly formulate and contrast the FedPM update rule with this hypothetical alternative.

The global model update in FedPM, by expanding Eq. (10), is given by:
\begin{align}
&\vtheta^{(t+1)} = \frac{1}{N} \sum_{i=1}^N (\mP^{(t)})^{-1} \mP_i^{(t,K-1)} \vtheta_i^{(t,K)} \notag \\
&= \frac{1}{N} \sum_{i=1}^N (\mP^{(t)})^{-1} \mP_i^{(t,K-1)} \notag \\
& \hspace{70pt}\left(\vtheta^{(t)} - \eta \sum_{k=0}^{K-1} (\mP_i^{(t,k)})^{-1} \nabla f_i(\vtheta_i^{(t,k)}) \right) \notag \\
&= \vtheta^{(t)} - \frac{\eta}{N} \sum_{i=1}^N (\mP^{(t)})^{-1} \mP_i^{(t,K-1)} \notag\\
&\hspace{100pt}\sum_{k=0}^{K-1} (\mP_i^{(t,k)})^{-1} \nabla f_i(\vtheta_i^{(t,k)}). \label{eq:fedpm_expanded}
\end{align}
This formulation shows that the global update incorporates a weighted sum of all local updates performed during the round. Crucially, each term in the inner sum, $(\mP_i^{(t,k)})^{-1} \nabla f_i(\vtheta_i^{(t,k)})$, is a locally preconditioned gradient from a specific local step $k$. The server then applies a preconditioned mixing, using both the global preconditioner $\mP^{(t)}$ and the clients' final local preconditioners $\mP_i^{(t,K-1)}$, to this entire history of local updates.

In contrast, a method that performs $K-1$ local second-order updates and concludes with a single SOGM-like global update would have the following form:
\begin{equation}
\vtheta^{(t+1)} = \vtheta^{(t)} - \eta (\mP^{(t)})^{-1} \left( \frac{1}{N} \sum_{i=1}^N \nabla f_i(\vtheta_i^{(t,K-1)}) \right). \label{eq:sogm_multi}
\end{equation}
Here, the global preconditioner $(\mP^{(t)})^{-1}$ is applied to the average of local gradients, but these gradients are evaluated \textit{only} at the final local iterate $\vtheta_i^{(t,K-1)}$. This hypothetical update completely disregards the trajectory of intermediate local updates and the sequence of local preconditioners $(\mP_i^{(t,k)})$ used to generate them.

\begin{figure*}[t]
    \centering
    \includegraphics[width=0.75\paperwidth]{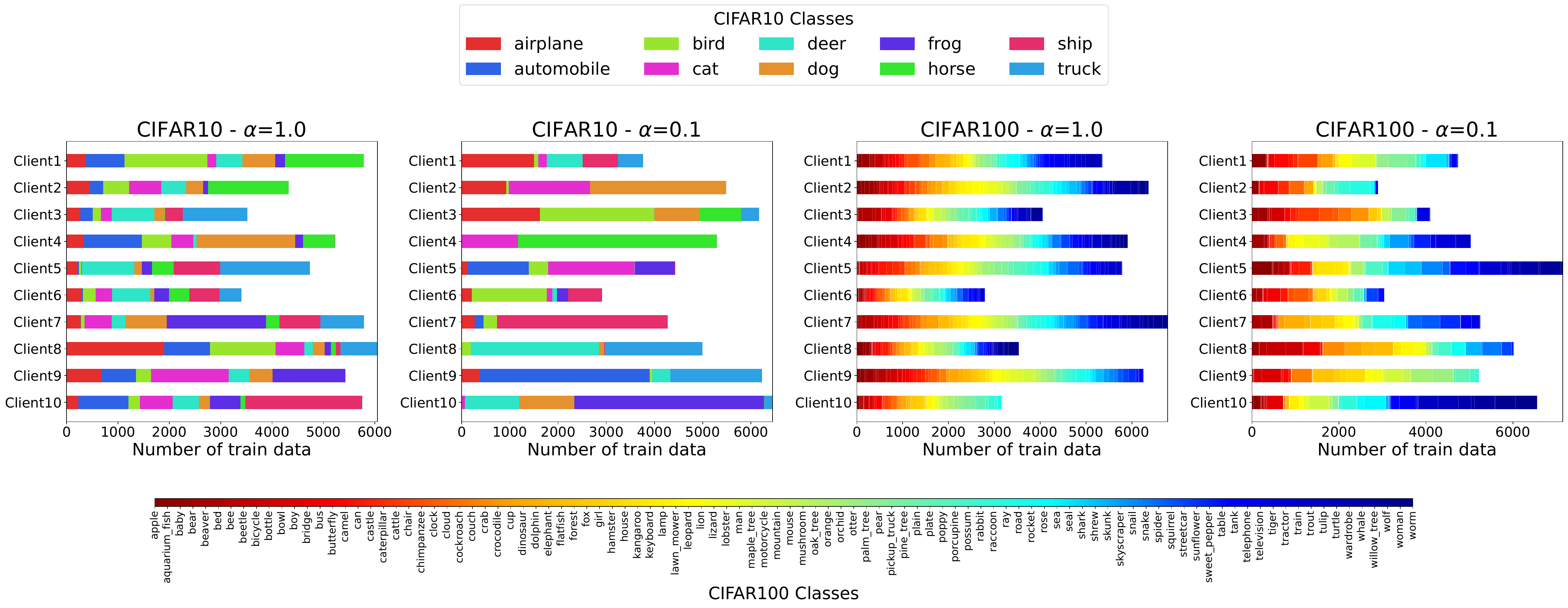}
    \caption{Data distributions of CIFAR10 and CIFAR100 under $\alpha$ = 1.0 and $\alpha$ = 0.1 for a single random seed. Different colors represent different classes.}
    \label{fig:data_distribution}
\end{figure*}

Comparing Eq. (\ref{eq:fedpm_expanded}) and Eq. (\ref{eq:sogm_multi}) makes it clear that the two update rules are not equivalent. The key distinction lies in how historical information from the local training phase is aggregated and utilized. FedPM's preconditioned mixing provides a more sophisticated aggregation that retains information from the entire local optimization path. We argue that this idea of preconditioned mixing, though simple in form, represents a meaningful and previously overlooked design choice in the context of second-order FL optimization.

\section{Detailed Settings of Numerical experiments (Sec. \ref{sec4})}
\label{secap:settings}
\subsection{Computing resources}
We conducted our experiments using two different computing environments: Test 1, which involved a single-GPU setup, and Test 2, which utilized a multi-node distributed system. Below, we provide detailed specifications for both setups.

In Test 1, we used a machine equipped with an NVIDIA GeForce RTX 2080 Ti GPU and an Intel Xeon Silver 4215 CPU with 8 cores. The system had 45 GB of RAM and ran on Ubuntu 24.04. The software environment included Python 3.10.9, PyTorch 2.1.0, and CUDA 11.8.

For Test 2, we used a distributed system consisting of three compute nodes, each equipped with four NVIDIA H100 SXM5 GPUs with 94GB of HBM2e memory, connected via NVLink. Each node had two AMD EPYC 9654 CPUs clocked at 2.4 GHz, with 96 cores per CPU, and 768 GB of DDR5-4800 RAM. The nodes were interconnected using four InfiniBand NDR200 links, each providing 200 Gbps bandwidth. Each client and the server in experiments were allocated a single GPU to simulate FL setting. This setup ran on Red Hat Enterprise Linux Server 9.3 with the same software environment as Test 1: Python 3.10.9, PyTorch 2.1.0, and CUDA 11.8.

\subsection{Data distributions}
In Test 2 in Sec. \ref{subsec:test2}, we partitioned data among clients following the approach described in \cite{vogels2021relaysum}. This method utilizes a Dirichlet distribution to control the level of data heterogeneity across clients. The heterogeneity is governed by the Dirichlet concentration parameter $\alpha$, where a smaller $\alpha$ results in a more imbalanced and heterogeneous data distribution, while a larger $\alpha$ yields a more uniform allocation. The data distributions for CIFAR10 and CIFAR100 are depicted in Figure \ref{fig:data_distribution}.

% \subsection{Hyperparameter tuning}
% \label{secap:tuning}
% In Test 2, we performed systematic hyperparameter tuning to ensure optimal performance across different settings. The learning rate was selected from \{0.8, 0.5, 0.3, 0.1, 0.05, 0.03, 0.01, 0.005\}, ensuring that the best value was not at the edge of the search space. The maximum norm value for gradient clipping was tuned from \{1.0, no clipping\}. The weight decay parameter was chosen from \{0.01, 0.0001, 0.0 (no weight decay)\}. For FedAvgM, we tuned the momentum from \{0.99, 0.9, 0.7\}. In the case of FedProx, the regularization coefficient $\mu$ was tuned from \{0.01, 0.001, 0.0001\}. For LocalNewton and FedPM, the damping term, which stabilizes the inverse computation by adding a small value to the diagonal elements of the preconditioner matrix, was selected from \{1.0, 0.01, 0.0001\}. The batch size was fixed to be 64 for all experiments.

% Hyperparameter tuning was conducted using three different random seeds for each experiment setting, and the best hyperparameters were determined based on the average of the highest test accuracy of the global model achieved during training for each random seed.
\subsection{Hyperparameter tuning}

We performed systematic hyperparameter tuning to ensure optimal performance for each method across the different experimental settings.

\noindent
For \textbf{all experiments}:
\begin{itemize}
    \item The momentum for \textbf{FedAvgM} was tuned from $\{0.7, 0.9, 0.99\}$.
    \item The global learning rate for \textbf{SCAFFOLD} was fixed to 1.0.
    \item For \textbf{FedAdam}, $\beta_1$, $\beta_2$, and $\tau$ were fixed at 0.9, 0.99, and 0.001, respectively, while the server learning rate was tuned from $\{0.05, 0.03, 0.01\}$.
\end{itemize}

For \textbf{Test 1}, which involved a strongly convex model, we evaluated using only a single seed, as there was no randomness in the experiment due to the fixed data distribution. The learning rate was tuned exhaustively for each method to find the optimal value. For specific second-order methods in this test:
\begin{itemize}
    \item The sketch size for \textbf{FedNS} was set to be the same as the model dimension.
    \item For \textbf{FedNL}, no compression was applied, and the Hessian learning rate was fixed to 1.
    \item No damping term was added to the Hessian, as it was guaranteed to be positive definite.
\end{itemize}

\begin{figure*}[t]
\centering
\includegraphics[width=0.75\paperwidth]{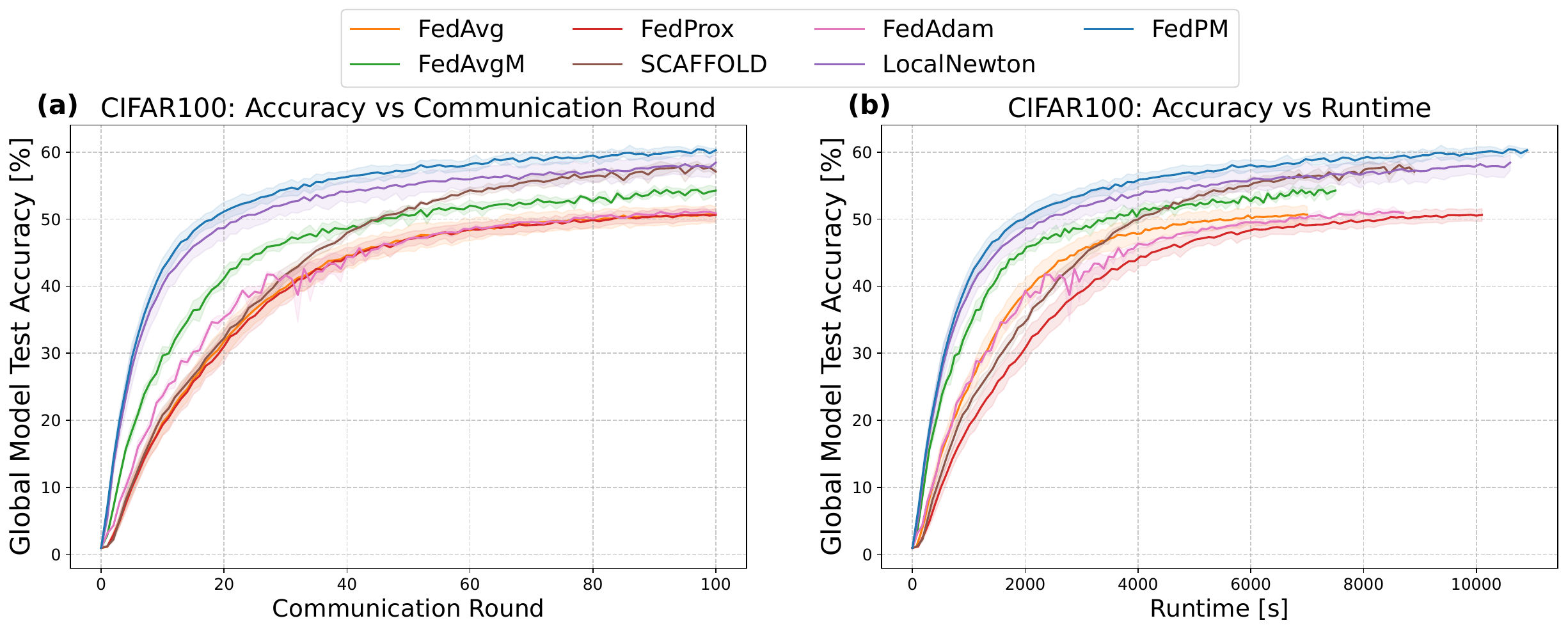}
\caption{Convergence curves for Test 2 using the global model on CIFAR100 with a heterogeneity level of $\alpha=0.1$ and 5 local epochs. (a) depicts test accuracy against communication rounds, while (b) shows test accuracy against runtime. The shaded areas represent one standard deviation across three random seeds. These results correspond to the settings for CIFAR100 with $\alpha=0.1$ reported in Table 3.}
\label{fig:appendix_cifar100_curves}
\end{figure*}

\begin{figure}
    \centering
    \includegraphics[width=\linewidth]{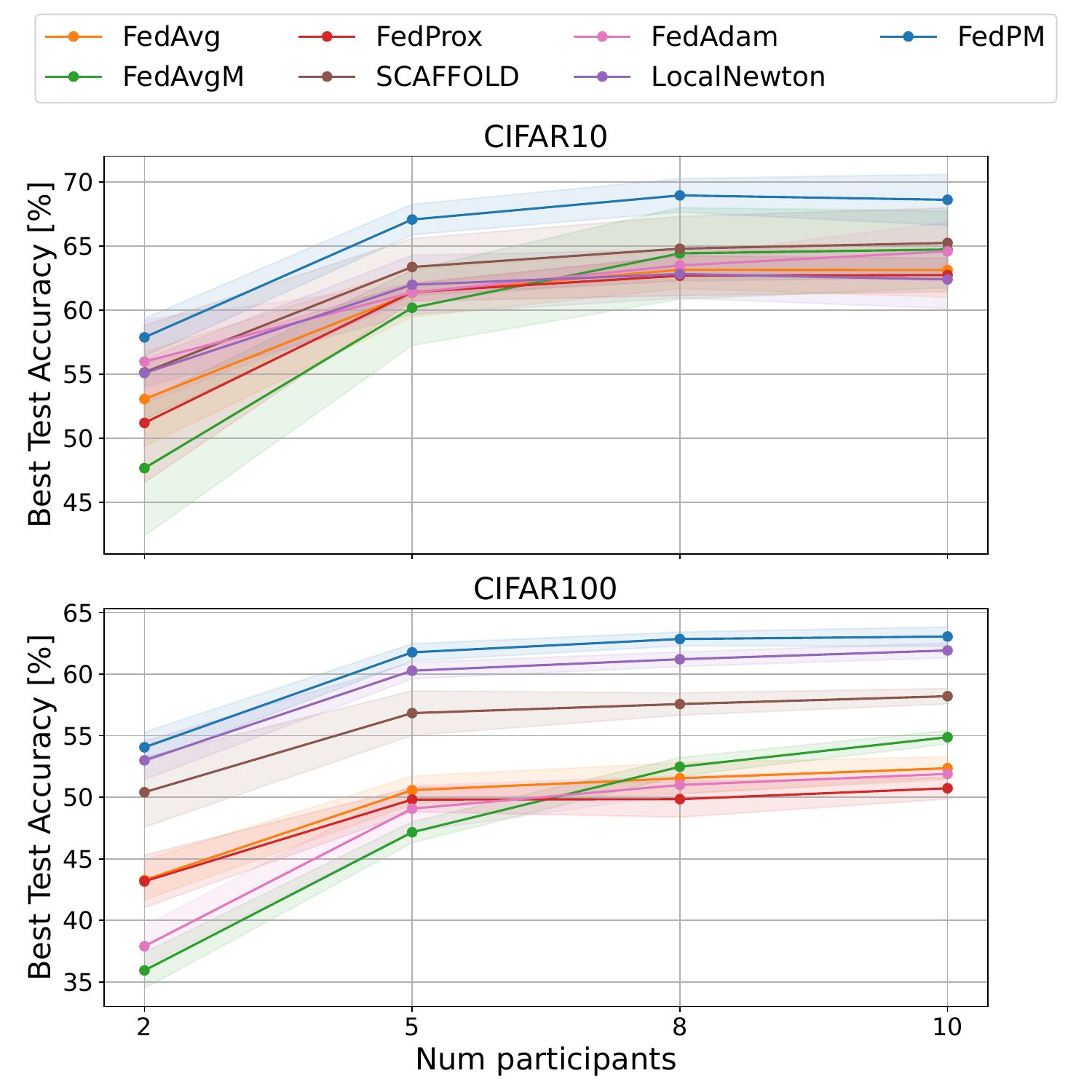}
    \caption{Investigation of the impact of client sampling on test accuracy for CIFAR10 (upper plot and CIFAR100 (lower plot) with $\alpha$ = 0.1 and 5 inner epochs. The plot shows the highest average of best test accuracies across three random seeds for different numbers of participating clients. The shaded area represents one standard deviation.}
    \label{fig:client_sampling}
\end{figure}

For \textbf{Test 2}, which involved non-convex deep learning models, we conducted a more extensive search. The batch size was fixed at 64 for all experiments.
\begin{itemize}
    \item \textbf{Learning Rate}: Selected from $\{0.8, 0.5, 0.3, 0.1, 0.05,$\\$ 0.03, 0.01\}$.
    \item \textbf{Gradient Clipping}: The maximum norm value was tuned from $\{1.0, \text{no clipping}\}$.
    \item \textbf{Weight Decay}: Chosen from $\{0.01, 0.0001,$\\$ 0.0 \text{ (no weight decay)}\}$.
    \item \textbf{FedProx}: The regularization coefficient $\mu$ was tuned from $\{0.01, 0.001, 0.0001\}$.
    \item \textbf{LocalNewton and FedPM}: The damping term, used to stabilize matrix inversion, was selected from $\{1.0, 0.01, 0.0001\}$.
\end{itemize}

Hyperparameter tuning for Test 2 was conducted using three different random seeds for each experimental setting. The best hyperparameters were determined based on the average of the highest test accuracy achieved by the global model during training for each random seed. The final tuned hyperparameters are described from Table \ref{tab:hp_cifar10_a1_0} to Table \ref{tab:hp_foof_samples}.

\begin{table*}
\centering
\setlength{\tabcolsep}{1mm}
{\fontsize{9}{\baselineskip}\selectfont
\begin{tabular}{lcccccc}
\toprule
\textbf{Method} & \textbf{Learning Rate} & \textbf{Grad. Clip} & \textbf{Weight Decay} & \textbf{Momentum} & \textbf{Damping Term} & \textbf{Server LR} \\
\midrule
FedAvg & 0.1 & No & 0.0001 & - & - & - \\
FedAvgM & 0.1 & 1.0 & 0.0001 & 0.9 & - & - \\
FedProx & 0.3 & 1.0 & 0.001 ($\mu$) & - & - & - \\
SCAFFOLD & 0.1 & No & 0.0001 & - & - & - \\
FedAdam & 0.05 & No & 0.0001 & - & - & 0.03 \\
LocalNewton & 0.3 & No & 0.0001 & - & 1.0 & - \\
FedPM & 0.3 & No & 0.0001 & - & 1.0 & - \\
\bottomrule
\end{tabular}}
\caption{Final hyperparameters for CIFAR10, $\alpha$=1.0 (Table \ref{tab:accuracy_table}).}
\label{tab:hp_cifar10_a1_0}
\end{table*}

\begin{table*}
\centering
\setlength{\tabcolsep}{1mm}
{\fontsize{9}{\baselineskip}\selectfont
\begin{tabular}{lcccccc}
\toprule
\textbf{Method} & \textbf{Learning Rate} & \textbf{Grad. Clip} & \textbf{Weight Decay} & \textbf{Momentum} & \textbf{Damping Term} & \textbf{Server LR} \\
\midrule
FedAvg & 0.1 & No & 0.0 & - & - & - \\
FedAvgM & 0.1 & 1.0 & 0.0001 & 0.9 & - & - \\
FedProx & 0.05 & No & 0.001 ($\mu$) & - & - & - \\
SCAFFOLD & 0.1 & No & 0.0001 & - & - & - \\
FedAdam & 0.05 & No & 0.0001 & - & - & 0.03 \\
LocalNewton & 0.3 & 1.0 & 0.0 & - & 0.0001 & - \\
FedPM & 0.5 & 1.0 & 0.0001 & - & 1.0 & - \\
\bottomrule
\end{tabular}}
\caption{Final hyperparameters for CIFAR10, $\alpha$=0.1 (Table \ref{tab:accuracy_table} and Figure \ref{fig:curve}).}
\label{tab:hp_cifar10_a0_1}
\end{table*}

\begin{table*}
\centering
\setlength{\tabcolsep}{1mm}
{\fontsize{9}{\baselineskip}\selectfont
\begin{tabular}{lcccccc}
\toprule
\textbf{Method} & \textbf{Learning Rate} & \textbf{Grad. Clip} & \textbf{Weight Decay} & \textbf{Momentum} & \textbf{Damping Term} & \textbf{Server LR} \\
\midrule
FedAvg & 0.3 & 1.0 & 0.0001 & - & - & - \\
FedAvgM & 0.1 & 1.0 & 0.0001 & 0.9 & - & - \\
FedProx & 0.3 & 1.0 & 0.01 ($\mu$) & - & - & - \\
SCAFFOLD & 0.3 & 1.0 & 0.0001 & - & - & - \\
FedAdam & 0.03 & No & 0.0001 & - & - & 0.01 \\
LocalNewton & 0.3 & No & 0.0001 & - & 1.0 & - \\
FedPM & 0.3 & No & 0.0001 & - & 1.0 & - \\
\bottomrule
\end{tabular}}
\caption{Final hyperparameters for CIFAR100, $\alpha$=1.0 (Table \ref{tab:accuracy_table}).}
\label{tab:hp_cifar100_a1_0}
\end{table*}

\begin{table*}
\centering
\setlength{\tabcolsep}{1mm}
{\fontsize{9}{\baselineskip}\selectfont
\begin{tabular}{lcccccc}
\toprule
\textbf{Method} & \textbf{Learning Rate} & \textbf{Grad. Clip} & \textbf{Weight Decay} & \textbf{Momentum} & \textbf{Damping Term} & \textbf{Server LR} \\
\midrule
FedAvg & 0.3 & No & 0.0001 & - & - & - \\
FedAvgM & 0.1 & 1.0 & 0.0001 & 0.9 & - & - \\
FedProx & 0.3 & No & 0.001 ($\mu$) & - & - & - \\
SCAFFOLD & 0.3 & No & 0.0001 & - & - & - \\
FedAdam & 0.01 & No & 0.0001 & - & - & 0.03 \\
LocalNewton & 0.5 & 1.0 & 0.0 & - & 1.0 & - \\
FedPM & 0.5 & 1.0 & 0.0001 & - & 1.0 & - \\
\bottomrule
\end{tabular}}
\caption{Final hyperparameters for CIFAR100, $\alpha$=0.1 (Table \ref{tab:accuracy_table} and Figure \ref{fig:appendix_cifar100_curves}).}
\label{tab:hp_cifar100_a0_1}
\end{table*}

\begin{table*}
\centering
\setlength{\tabcolsep}{1mm}
{\fontsize{9}{\baselineskip}\selectfont
\begin{tabular}{lcccccc}
\toprule
\textbf{Method} & \textbf{Learning Rate} & \textbf{Grad. Clip} & \textbf{Weight Decay} & \textbf{Momentum} & \textbf{Damping Term} & \textbf{Server LR} \\
\midrule
FedAvg & 0.05 & No & 0.0001 & - & - & - \\
FedAvgM & 0.1 & 1.0 & 0.0001 & 0.99 & - & - \\
FedProx & 0.1 & No & 0.001 ($\mu$) & - & - & - \\
SCAFFOLD & 0.3 & 1.0 & 0.0001 & - & - & - \\
FedAdam & 0.05 & No & 0.0001 & - & - & 0.05 \\
LocalNewton & 0.3 & 1.0 & 0.0001 & - & 1.0 & - \\
FedPM & 0.3 & 1.0 & 0.0001 & - & 0.01 & - \\
\bottomrule
\end{tabular}}
\caption{Final hyperparameters for CIFAR10 ($\alpha$=0.1) with 1 local epoch (Figure \ref{fig:inner_epoch_result}).}
\label{tab:hp_cifar10_e1}
\end{table*}

\begin{table*}
\centering
\setlength{\tabcolsep}{1mm}
{\fontsize{9}{\baselineskip}\selectfont
\begin{tabular}{lcccccc}
\toprule
\textbf{Method} & \textbf{Learning Rate} & \textbf{Grad. Clip} & \textbf{Weight Decay} & \textbf{Momentum} & \textbf{Damping Term} & \textbf{Server LR} \\
\midrule
FedAvg & 0.05 & No & 0.0001 & - & - & - \\
FedAvgM & 0.1 & 1.0 & 0.0001 & 0.9 & - & - \\
FedProx & 0.1 & 1.0 & 0.001 ($\mu$) & - & - & - \\
SCAFFOLD & 0.1 & No & 0.0001 & - & - & - \\
FedAdam & 0.03 & No & 0.0001 & - & - & 0.05 \\
LocalNewton & 0.3 & 1.0 & 0.0 & - & 0.01 & - \\
FedPM & 0.3 & No & 0.0001 & - & 1.0 & - \\
\bottomrule
\end{tabular}}
\caption{Final hyperparameters for CIFAR10 ($\alpha$=0.1) with 10 local epochs (Figure \ref{fig:inner_epoch_result}).}
\label{tab:hp_cifar10_e10}
\end{table*}

\begin{table*}
\centering
\setlength{\tabcolsep}{1mm}
{\fontsize{9}{\baselineskip}\selectfont
\begin{tabular}{lcccc}
\toprule
& \multicolumn{4}{c}{\textbf{Number of Participating Clients per Round}} \\
\cmidrule(lr){2-5}
\textbf{Method} & \textbf{2} & \textbf{5} & \textbf{8} & \textbf{10} \\
\midrule
FedAvg & (0.05, No, 0.0001) & (0.05, No, 0.0) & (0.05, No, 0.0001) & (0.1, No, 0.0) \\
FedAvgM & (0.1, 1.0, 0.0001, 0.7) & (0.1, 1.0, 0.0001, 0.7) & (0.1, 1.0, 0.0001, 0.9) & (0.1, 1.0, 0.0001, 0.9) \\
FedProx & (0.1, No, 0.01) & (0.1, No, 0.001) & (0.1, No, 0.001) & (0.05, No, 0.001) \\
SCAFFOLD & (0.1, 1.0, 0.0001) & (0.1, No, 0.0001) & (0.1, No, 0.0001) & (0.1, No, 0.0001) \\
LocalNewton & (0.1, No, 0.0001, 1.0) & (0.5, No, 0.0, 1.0) & (0.5, 1.0, 0.0, 0.01) & (0.3, 1.0, 0.0, 0.0001) \\
FedPM & (0.3, 1.0, 0.0001, 1.0) & (0.3, 1.0, 0.0001, 1.0) & (0.3, 1.0, 0.0001, 1.0) & (0.5, 1.0, 0.0001, 1.0) \\
\bottomrule
\end{tabular}}
\caption{Final hyperparameters for client sampling experiments on CIFAR10 (Figure \ref{fig:client_sampling}). The values in parentheses represent (Learning Rate, Grad. Clip, Weight Decay/$\mu$). For FedAvgM, the fourth value is the momentum. For LocalNewton and FedPM, the fourth value is the damping term.}
\label{tab:hp_clientsamp_cifar10}
\end{table*}

\begin{table*}
\centering
\setlength{\tabcolsep}{1mm}
{\fontsize{9}{\baselineskip}\selectfont
\begin{tabular}{lcccc}
\toprule
& \multicolumn{4}{c}{\textbf{Number of Participating Clients per Round}} \\
\cmidrule(lr){2-5}
\textbf{Method} & \textbf{2} & \textbf{5} & \textbf{8} & \textbf{10} \\
\midrule
FedAvg & (0.3, No, 0.0001) & (0.3, No, 0.0001) & (0.5, No, 0.0001) & (0.3, No, 0.0001) \\
FedAvgM & (0.1, No, 0.0, 0.7) & (0.1, 1.0, 0.0, 0.7) & (0.1, 1.0, 0.0001, 0.9) & (0.1, 1.0, 0.0001, 0.9) \\
FedProx & (0.3, No, 0.001) & (0.3, No, 0.001) & (0.3, No, 0.001) & (0.3, No, 0.001) \\
SCAFFOLD & (0.5, No, 0.0001) & (0.5, No, 0.0001) & (0.3, No, 0.0001) & (0.3, No, 0.0001) \\
LocalNewton & (0.5, 1.0, 0.0, 1.0) & (0.5, 1.0, 0.0, 1.0) & (0.5, 1.0, 0.0, 1.0) & (0.5, 1.0, 0.0, 1.0) \\
FedPM & (0.5, 1.0, 0.0, 1.0) & (0.5, 1.0, 0.0, 1.0) & (0.5, 1.0, 0.0, 1.0) & (0.5, 1.0, 0.0001, 1.0) \\
\bottomrule
\end{tabular}}
\caption{Final hyperparameters for client sampling experiments on CIFAR100 (Figure \ref{fig:client_sampling}). The values in parentheses represent (Learning Rate, Grad. Clip, Weight Decay/$\mu$). For FedAvgM, the fourth value is the momentum. For LocalNewton and FedPM, the fourth value is the damping term.}
\label{tab:hp_clientsamp_cifar100}
\end{table*}

\begin{table*}
\centering
\setlength{\tabcolsep}{1mm}
{\fontsize{9}{\baselineskip}\selectfont
\begin{tabular}{lcccccc}
\toprule
\textbf{Method} & \textbf{Learning Rate} & \textbf{Grad. Clip} & \textbf{Weight Decay} & \textbf{Momentum} & \textbf{Damping Term} & \textbf{Server LR} \\
\midrule
FedAvg & 0.1 & No & 0.0001 & - & - & - \\
FedAvgM & 0.5 & 1.0 & 0.0001 & 0.7 & - & - \\
FedProx & 0.1 & No & 0.01 ($\mu$) & - & - & - \\
SCAFFOLD & 0.5 & 1.0 & 0.0001 & - & - & - \\
FedAdam & 0.05 & No & 0.0001 & - & - & 0.03 \\
LocalNewton & 0.5 & 1.0 & 0.0001 & - & 1.0 & - \\
FedPM & 0.5 & 1.0 & 0.0001 & - & 1.0 & - \\
\bottomrule
\end{tabular}}
\caption{Final hyperparameters for FEMNIST (Table \ref{tab:femnist}).}
\label{tab:hp_femnist}
\end{table*}

\begin{table*}
\centering
\setlength{\tabcolsep}{1mm}
{\fontsize{9}{\baselineskip}\selectfont
\begin{tabular}{llcccc}
\toprule
\textbf{Dataset} & \textbf{\# Samples} & \textbf{Learning Rate} & \textbf{Grad. Clip} & \textbf{Weight Decay} & \textbf{Damping Term} \\
\midrule
CIFAR10 & 64 & 0.5 & 1.0 & 0.0001 & 1.0 \\
& 256 & 0.5 & No & 0.0001 & 1.0 \\
& 1024 & 0.5 & No & 0.0 & 1.0 \\
& full & 0.5 & 1.0 & 0.0001 & 1.0 \\
\midrule
CIFAR100 & 64 & 0.3 & 1.0 & 0.0001 & 1.0 \\
& 256 & 0.3 & 1.0 & 0.0001 & 1.0 \\
& 1024 & 0.3 & 1.0 & 0.0001 & 1.0 \\
& full & 0.5 & 1.0 & 0.0001 & 1.0 \\
\bottomrule
\end{tabular}}
\caption{Final hyperparameters for FedPM with varying sample sizes for FOOF computation (Figure \ref{fig:samples_runtime}).}
\label{tab:hp_foof_samples}
\end{table*}

\subsection{Statistical Significance}
We use the Mann-Whitney U test \cite{mann1947test} to test the statistical significance of FedPM against other methods across three seeds. Table \ref{tab:p_values} summarizes the results by showing the maximum p-value obtained when comparing the best test accuracy of FedPM against that of every other method in each setting. A lower value indicates a more significant performance improvement by FedPM.

\begin{table*}[h!]
\centering
\setlength{\tabcolsep}{1mm}
% {\fontsize{9}{\baselineskip}\selectfont
\begin{tabular}{lc}
\toprule
\textbf{Experimental Setting} & \textbf{Maximum p-value} \\
\midrule
\multicolumn{2}{l}{\textbf{Main Experiments (Table \ref{tab:accuracy_table})}} \\
CIFAR10, $\alpha$=1.0 & 0.20 \\
CIFAR100, $\alpha$=1.0 & 0.80 (0.10 excluding LocalNewton) \\
% CIFAR100, $\alpha$=1.0 without LocalNewton & 0.10 \\
CIFAR10, $\alpha$=0.1 (Fig. \ref{fig:curve}) & 0.20 \\
CIFAR100, $\alpha$=0.1 (Fig. \ref{fig:appendix_cifar100_curves}) & 0.20 \\
\midrule
\multicolumn{2}{l}{\textbf{Varying Local Epochs (Figure \ref{fig:inner_epoch_result})}} \\
CIFAR10, 1 local epoch & 0.50 \\
CIFAR10, 10 local epochs & 0.20 \\
\midrule
\multicolumn{2}{l}{\textbf{Client Sampling (Figure \ref{fig:client_sampling})}} \\
CIFAR10, 2 participants & 0.20 \\
CIFAR10, 5 participants & 0.05 \\
CIFAR10, 8 participants & 0.05 \\
CIFAR100, 2 participants & 0.35 \\
CIFAR100, 5 participants & 0.10 \\
CIFAR100, 8 participants & 0.05 \\
\bottomrule
\end{tabular}
\caption{Maximum p-value from the Mann-Whitney U test when comparing FedPM against all other methods for each experimental setting.}
\label{tab:p_values}
\end{table*}

\section{Additional experimental results}
\label{secap:addexp}
\subsection{Convergence Curves on CIFAR100 classification}

In connection with Sec. \ref{subsec:test2}, this section provides the convergence curves for the experiments on the CIFAR100 classification. Figure~\ref{fig:appendix_cifar100_curves} illustrates the test accuracy of the global model versus both communication rounds and runtime. The experimental setup corresponds to the results presented in Table \ref{tab:accuracy_table} for CIFAR100 with a high level of data heterogeneity ($\alpha=0.1$) and 5 local update epochs per round.

The results shown in the figure reinforce the conclusions drawn from the CIFAR10 experiments. FedPM not only converges to a significantly higher test accuracy but also does so more rapidly than the other methods, both in terms of communication rounds and wall-clock time. This further validates the effectiveness of the preconditioned mixing strategy in complex, non-convex optimization scenarios under challenging heterogeneous data distributions.

\subsection{Empirical investigation of negative impacts regarding client sampling}

In FL, client sampling is generally employed. However, we omit it throughout the main paper to hold simple notation and theoretical analysis. 
In this section, we empirically investigate the negative impacts of client sampling on the performance of various FL methods. We used the same experimental settings as in the main paper experiments; the number of clients participating in training per round varied from $2$ to $10$. The experimental results are summarized in Figure~\ref{fig:client_sampling}.

As expected, all methods exhibited performance degradation as the number of participating clients decreased. Reducing the number of participating clients per round leads to a noticeable decline in test accuracy across all methods. This degradation can be attributed to the increased variance in model updates caused by less representative client selections, which slows down convergence and impacts generalization. Among the evaluated methods, our FedPM demonstrated the highest robustness to client sampling. Even when only two clients participated per round, FedPM achieved a higher accuracy compared to other methods under the same conditions. This suggests that FedPM effectively mitigates the adverse effects of limited client participation by stabilizing model updates using preconditioned mixing.

% \subsection{Communication and Computation Analysis}
% A key concern for second-order methods is their computational overhead. While FedPM requires transmitting preconditioners, increasing communication load, its superior convergence can lead to faster overall training time. To quantify this, we measured the computation and communication time per round on the CIFAR-10 task (details in Test 2).

% \begin{table}[t]
% \centering
% \begin{tabular}{lcc}
% \toprule
% \textbf{Method} & \textbf{Computation Time [s]} & \textbf{Communication Time [s]} \\
% \midrule
% FedAvg & 4827 & 553 \\
% LocalNewton & 6075 & 506 \\
% FedPM (ours) & 5953 & 832 \\
% \bottomrule
% \end{tabular}
% \caption{Quantitative trade-off analysis. Time is measured in seconds for a full 100-round run on the CIFAR-10 setup.}
% \label{tab:tradeoff}
% \end{table}

\begin{figure}[t]
    \centering
    \includegraphics[width=\linewidth]{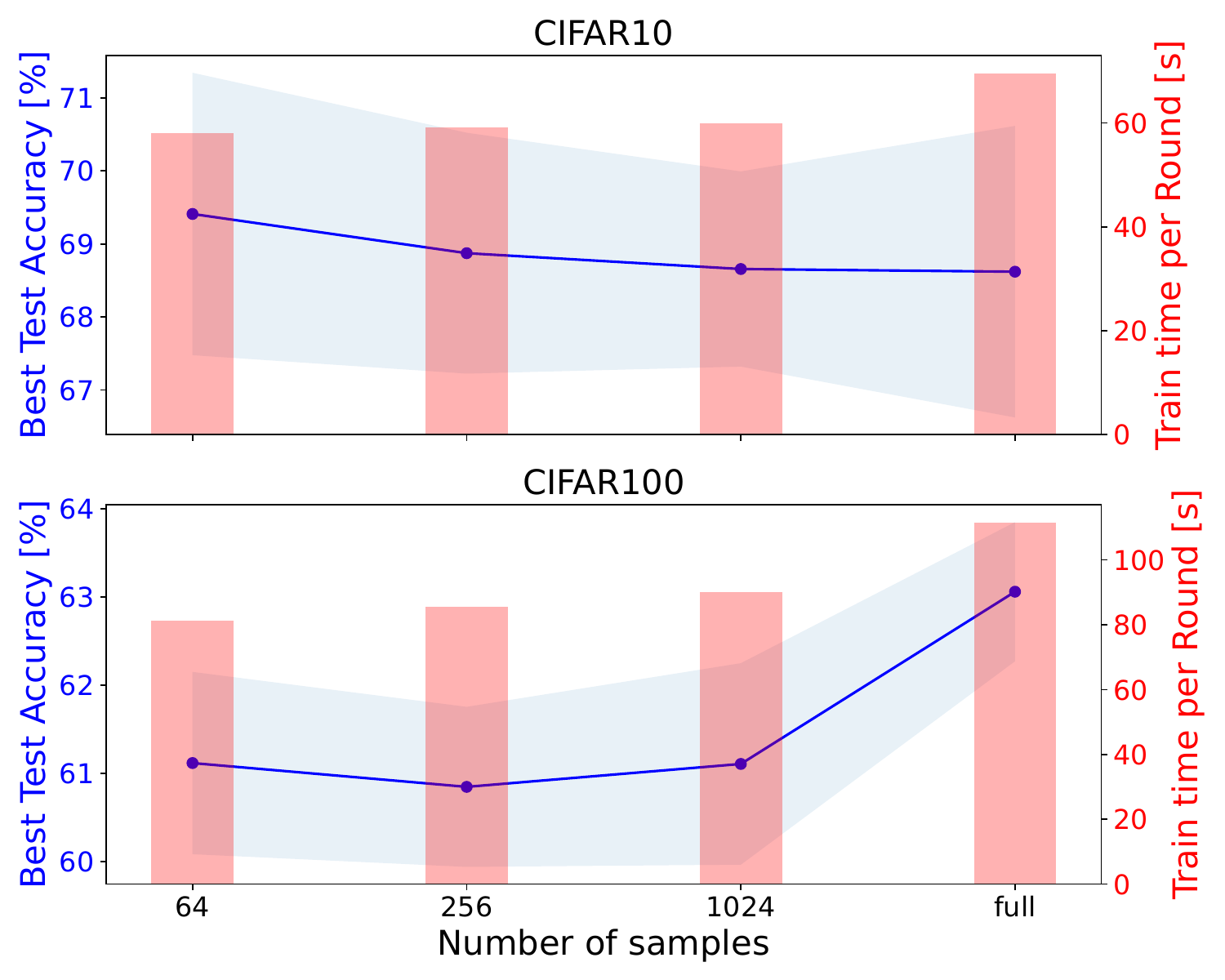}
    \caption{Effect of varying the number of local data samples used to compute FOOF matrices in FedPM. The line plot shows the best test accuracy (mean $\pm$ one standard deviation over three random seeds). The bar plot represents the total training time across all clients per round averaged over all runs. Experiments were conducted using 64, 256, and 1024 samples, in addition to using the full local dataset as done in Figure \ref{fig:curve}.}
    \label{fig:samples_runtime}
\end{figure}

\subsection{Evaluation on a Real-World Federated Dataset}
We evaluated FedPM on the FEMNIST dataset \cite{caldas2019leaf}. FEMNIST is derived from the Extended MNIST dataset and is naturally partitioned by the 3,597 different writers of the characters, creating a realistic, user-level non-IID data distribution. We sampled 10 clients per round and ran for 100 rounds with 5 inner epochs between them.
Table \ref{tab:femnist} shows the best test accuracy achieved. In this challenging, real-world setting, FedPM again demonstrates superior performance, outperforming strong baselines. This result validates that the benefits of preconditioned mixing are not limited to synthetic partitions but extend to practical, highly heterogeneous federated scenarios.

\begin{table}[t]
\centering
\begin{tabular}{lc}
\toprule
\textbf{Method} & \textbf{Best Test Accuracy [\%]} \\
\midrule
FedAvg & 75.83$\pm$1.90 \\
FedAvgM & 78.46$\pm$0.82 \\
FedProx & 76.42$\pm$0.43 \\
SCAFFOLD & 78.30$\pm$2.72 \\
FedAdam & 76.85$\pm$1.46 \\
LocalNewton & 77.36$\pm$1.12 \\
\textbf{FedPM (ours)} & \textbf{79.21}$\pm$1.69 \\
\bottomrule
\end{tabular}
\caption{Average best test accuracy of global parameter (over three random seeds) for each method on FEMNIST dataset, where local updates are performed for 5 epochs before each communication. The standard deviation across different seeds is also shown for each.}
\label{tab:femnist}
\end{table}

\subsection{Effect of Sample Size on FOOF Matrix Computation in FedPM}
The computational cost of second-order methods like FedPM is influenced by the number of local data samples used to compute the FOOF matrices. To assess its impact, we conducted experiments varying the number of samples while keeping all other settings identical to Figure \ref{fig:curve}. Specifically, we evaluated FedPM using $64$, $256$, and $1024$ samples for matrix computation, in addition to using the full local dataset as done in Figure \ref{fig:curve}.

Figure \ref{fig:samples_runtime} presents the results. The line plot shows the best test accuracy, averaged over three random seeds, with the shaded area indicating one standard deviation. The bar plot represents the total training time across all clients per round.

For CIFAR10 (upper plot), reducing the number of samples used for FOOF computation did not degrade performance, suggesting that even a small subset of data is sufficient to estimate the preconditioner effectively. In contrast, for CIFAR100 (lower plot), using the full local dataset achieved the highest accuracy, indicating that a larger number of samples contributes to a better approximation of second-order information. However, among the reduced sample sizes ($64$, $256$, and $1024$), the difference in accuracy was minimal, suggesting that a moderate reduction can significantly improve efficiency without a noticeable loss in performance.

These empirical findings indicate that the necessity of full dataset usage for FOOF computation depends on dataset complexity. While FedPM can maintain accuracy with fewer samples for simpler datasets like CIFAR10, leveraging the full dataset is more beneficial for complex datasets like CIFAR100.

\subsection{Profiling Results}
\label{secap:profiling}
We conducted empirical profiling of computational and communication costs. Table \ref{tab:profiling} summarizes the average per-round client training time, per-round communication time, and peak GPU memory usage for both CIFAR10 and CIFAR100 experiments.

\begin{table*}[t]
\centering
\begin{tabular}{lcccccc}
\toprule
& \multicolumn{3}{c}{\textbf{CIFAR10}} & \multicolumn{3}{c}{\textbf{CIFAR100}} \\
\cmidrule(lr){2-4} \cmidrule(lr){5-7}
\textbf{Method} & \begin{tabular}{@{}c@{}}Client Train\\Time (s/round)\end{tabular} & \begin{tabular}{@{}c@{}}Comm.\\Time (s/round)\end{tabular} & \begin{tabular}{@{}c@{}}Peak Memory\\(MB)\end{tabular} & \begin{tabular}{@{}c@{}}Client Train\\Time (s/round)\end{tabular} & \begin{tabular}{@{}c@{}}Comm.\\Time (s/round)\end{tabular} & \begin{tabular}{@{}c@{}}Peak Memory\\(MB)\end{tabular} \\
\midrule
FedAvg      & 48.27 & 5.53 & 23.80 & 55.47 & 14.03 & 450.71 \\
FedAvgM     & 49.48 & 5.52 & 24.05 & 60.53 & 13.98 & 497.01 \\
FedProx     & 53.01 & 5.49 & 23.80 & 84.47 & 13.95 & 529.62 \\
SCAFFOLD    & 49.44 & 11.06 & 24.56 & 60.52 & 27.82 & 620.92 \\
FedAdam     & 53.97 & 5.51 & 24.56 & 56.12 & 14.20 & 620.92 \\
LocalNewton & 60.75 & 5.06 & 24.31 & 90.35 & 14.13 & 542.44 \\
FedPM       & 59.53 & 8.32 & 37.07 & 89.57 & 20.99 & 542.44 \\
\bottomrule
\end{tabular}
\caption{Empirical profiling of average per-round client training time, communication time, and peak memory usage on CIFAR10 and CIFAR100. Measurements were averaged over all clients and rounds during the experiments for Figure \ref{fig:curve} and Figure \ref{fig:appendix_cifar100_curves}.}
\label{tab:profiling}
\end{table*}

\section{Limitations and Future Work}
\label{app:limitations}
The theoretical and empirical results presented in this paper establish FedPM as a robust second-order optimization method for Federated Learning. However, several avenues for future research remain to broaden its applicability and theoretical understanding.

\noindent
\textbf{Limitations in convergence analysis}: The current theoretical analysis of FedPM is restricted to scenarios involving a single local update ($K=1$) per communication round under which the global parameter update rule remains tractable. A primary direction for future work is to extend the convergence proof to accommodate multiple local updates ($K>1$). This requires addressing the analytical challenges posed by the drift in local preconditioners, which occurs as each client's model parameters diverge over local updates. Developing a theoretical framework that can bound the impact of this drift on the global convergence rate is essential for formally validating the use of FedPM in practical, communication-efficient scenarios. Furthermore, our analysis focuses on strongly convex models in line with most existing works \cite{safaryan2021fednl, elgabli2022fednew, li2024fedns}. As our experiments on DNNs demonstrate FedPM's empirical success in non-convex optimization, a crucial next step is to extend the theoretical analysis to non-convex settings.

\noindent
\textbf{Unexplored theoretical analysis of preconditioner approximations}: The paper outlines a practical preconditioner approximation (FOOF). However, the theoretical implications of this approximation, particularly their impact on convergence rates, have not been analyzed. Addressing this gap would enhance the theoretical guarantees and provide deeper insights into its efficacy.

\noindent
\textbf{Communication efficiency}: Although FedPM mitigates some communication inefficiencies, the additional transmission of preconditioners introduces communication overhead. Exploring strategies to further optimize communication without degrading convergence performance is a valuable direction for future work.

\end{document}